%% file: main_arxiv.tex
\documentclass{article}

\usepackage{microtype}
\usepackage{graphicx}
\usepackage{subcaption}
\usepackage{booktabs} % for professional tables

\usepackage{hyperref}
\usepackage{makecell}
\usepackage{enumitem}
\usepackage{multirow}
\usepackage{rotating}
\usepackage{wrapfig}

\usepackage[preprint]{icml2026}

\usepackage{amsmath}
\usepackage{amssymb}
\usepackage{mathtools}
\usepackage{amsthm}

\usepackage[capitalize,noabbrev]{cleveref}

\theoremstyle{plain}
\newtheorem{theorem}{Theorem}[section]

\theoremstyle{definition}

\theoremstyle{remark}

\usepackage[textsize=tiny]{todonotes}

\newcommand{\algname}{MixedDimKV}
\newcommand{\basicalg}{MixedDimKV}
\newcommand{\optalg}{MixedDimKV-H}

\icmltitlerunning{Beyond Token Eviction: Mixed-Dimension Budget Allocation for Efficient KV Cache Compression}

\begin{document}

\twocolumn[
  \icmltitle{Beyond Token Eviction: Mixed-Dimension Budget Allocation \\ for Efficient KV Cache Compression}

  % It is OKAY to include author information, even for blind submissions: the
  % style file will automatically remove it for you unless you've provided
  % the [accepted] option to the icml2026 package.

  % List of affiliations: The first argument should be a (short) identifier you
  % will use later to specify author affiliations Academic affiliations
  % should list Department, University, City, Region, Country Industry
  % affiliations should list Company, City, Region, Country

  % You can specify symbols, otherwise they are numbered in order. Ideally, you
  % should not use this facility. Affiliations will be numbered in order of
  % appearance and this is the preferred way.
  \icmlsetsymbol{equal}{*}

  \begin{icmlauthorlist}
    \icmlauthor{Ruijie Miao}{equal,PKU}
    \icmlauthor{Zhiming Wang}{equal,PKU}
    \icmlauthor{Wang Li}{equal,PKU}
    \icmlauthor{Shiwei Wu}{Seed}
    \icmlauthor{Shufan Liu}{Seed}
    \icmlauthor{Yanbing Jiang}{PKU}
    \icmlauthor{Tong Yang}{PKU}
    %\icmlauthor{}{sch}
    % \icmlauthor{Firstname8 Lastname8}{sch}
    % \icmlauthor{Firstname8 Lastname8}{yyy,comp}
    %\icmlauthor{}{sch}
    %\icmlauthor{}{sch}
  \end{icmlauthorlist}

  \icmlaffiliation{PKU}{Peking University}
  \icmlaffiliation{Seed}{ByteDance Seed}
  % \icmlaffiliation{sch}{School of ZZZ, Institute of WWW, Location, Country}

  \icmlcorrespondingauthor{Tong Yang}{yangtong@pku.edu.cn}
  % \icmlcorrespondingauthor{Firstname2 Lastname2}{first2.last2@www.uk}

  % You may provide any keywords that you find helpful for describing your
  % paper; these are used to populate the "keywords" metadata in the PDF but
  % will not be shown in the document
  \icmlkeywords{Machine Learning, ICML}

  \vskip 0.3in
]

% this must go after the closing bracket ] following \twocolumn[ ...

% This command actually creates the footnote in the first column listing the
% affiliations and the copyright notice. The command takes one argument, which
% is text to display at the start of the footnote. The \icmlEqualContribution
% command is standard text for equal contribution. Remove it (just {}) if you
% do not need this facility.

% Use ONE of the following lines. DO NOT remove the command.
% If you have no special notice, KEEP empty braces:
\printAffiliationsAndNotice{}  % no special notice (required even if empty)
% Or, if applicable, use the standard equal contribution text:
% \printAffiliationsAndNotice{\icmlEqualContribution}

\input{Body/abstract}

\input{Body/1.intro}

\input{Body/2.related}
\input{Body/3.preliminary}

\input{Body/4.solution}

\input{Body/5.implement}
\input{Body/6.experiment}

\input{Body/7.conclusion}

% \section*{Accessibility}

% Authors are kindly asked to make their submissions as accessible as possible
% for everyone including people with disabilities and sensory or neurological
% differences. Tips of how to achieve this and what to pay attention to will be
% provided on the conference website \url{http://icml.cc/}.

% \section*{Software and Data}

% If a paper is accepted, we strongly encourage the publication of software and
% data with the camera-ready version of the paper whenever appropriate. This can
% be done by including a URL in the camera-ready copy. However, \textbf{do not}
% include URLs that reveal your institution or identity in your submission for
% review. Instead, provide an anonymous URL or upload the material as
% ``Supplementary Material'' into the OpenReview reviewing system. Note that
% reviewers are not required to look at this material when writing their review.

% Acknowledgements should only appear in the accepted version.
% \section*{Acknowledgements}

% \textbf{Do not} include acknowledgements in the initial version of the paper
% submitted for blind review.

% If a paper is accepted, the final camera-ready version can (and usually should)
% include acknowledgements.  Such acknowledgements should be placed at the end of
% the section, in an unnumbered section that does not count towards the paper
% page limit. Typically, this will include thanks to reviewers who gave useful
% comments, to colleagues who contributed to the ideas, and to funding agencies
% and corporate sponsors that provided financial support.

% \newpage

\section*{Impact Statement}

This paper presents work whose goal is to advance the field of Machine Learning, and more specifically, the resource-efficiency of LLM inference. 
Our solution facilitates the efficient deployment of LLMs in resource-constrained environments, thereby democratizing access to cutting-edge AI for small organizations and independent researchers.
As a technical optimization for KV cache efficiency, this work does not introduce new ethical concerns or negative societal consequences regarding its deployment or use.

% In the unusual situation where you want a paper to appear in the
% references without citing it in the main text, use \nocite
% \nocite{langley00}

\bibliography{reference}
\bibliographystyle{icml2026}

%%%%%%%%%%%%%%%%%%%%%%%%%%%%%%%%%%%%%%%%%%%%%%%%%%%%%%%%%%%%%%%%%%%%%%%%%%%%%%%
%%%%%%%%%%%%%%%%%%%%%%%%%%%%%%%%%%%%%%%%%%%%%%%%%%%%%%%%%%%%%%%%%%%%%%%%%%%%%%%
% APPENDIX
%%%%%%%%%%%%%%%%%%%%%%%%%%%%%%%%%%%%%%%%%%%%%%%%%%%%%%%%%%%%%%%%%%%%%%%%%%%%%%%
%%%%%%%%%%%%%%%%%%%%%%%%%%%%%%%%%%%%%%%%%%%%%%%%%%%%%%%%%%%%%%%%%%%%%%%%%%%%%%%
\newpage
\appendix
\onecolumn

\input{Body/appendix}

%%%%%%%%%%%%%%%%%%%%%%%%%%%%%%%%%%%%%%%%%%%%%%%%%%%%%%%%%%%%%%%%%%%%%%%%%%%%%%%
%%%%%%%%%%%%%%%%%%%%%%%%%%%%%%%%%%%%%%%%%%%%%%%%%%%%%%%%%%%%%%%%%%%%%%%%%%%%%%%

\end{document}

%% file: Body/abstract.tex
\begin{abstract}
Key-value (KV) caching is widely used to accelerate transformer inference, but its memory cost grows linearly with input length, limiting long-context deployment. 
Existing token eviction methods reduce memory by discarding less important tokens, which can be viewed as a coarse form of dimensionality reduction that assigns each token either zero or full dimension.
We propose \algname{}, a mixed-dimension KV cache compression method that allocates dimensions to tokens at a more granular level, and \optalg{}, which further integrates head-level importance information.
Experiments on long-context benchmarks show that \algname{} outperforms prior KV cache compression methods that do not rely on head-level importance profiling. When equipped with the same head-level importance information, \optalg{} consistently outperforms HeadKV.
Notably, our approach achieves comparable performance to full attention on LongBench with only 6.25\% of the KV cache.
Furthermore, in the Needle-in-a-Haystack test, our solution maintains 100\% accuracy at a 50K context length while using as little as 0.26\% of the cache.
\end{abstract}

%% file: Body/1.intro.tex
\section{Introduction}

% LLM show ability on various tasks, especially long context
% KV cache is required, which has limitations in long contexts

\begin{figure*}[t]
    \centering
    \includegraphics[width=0.9\textwidth]{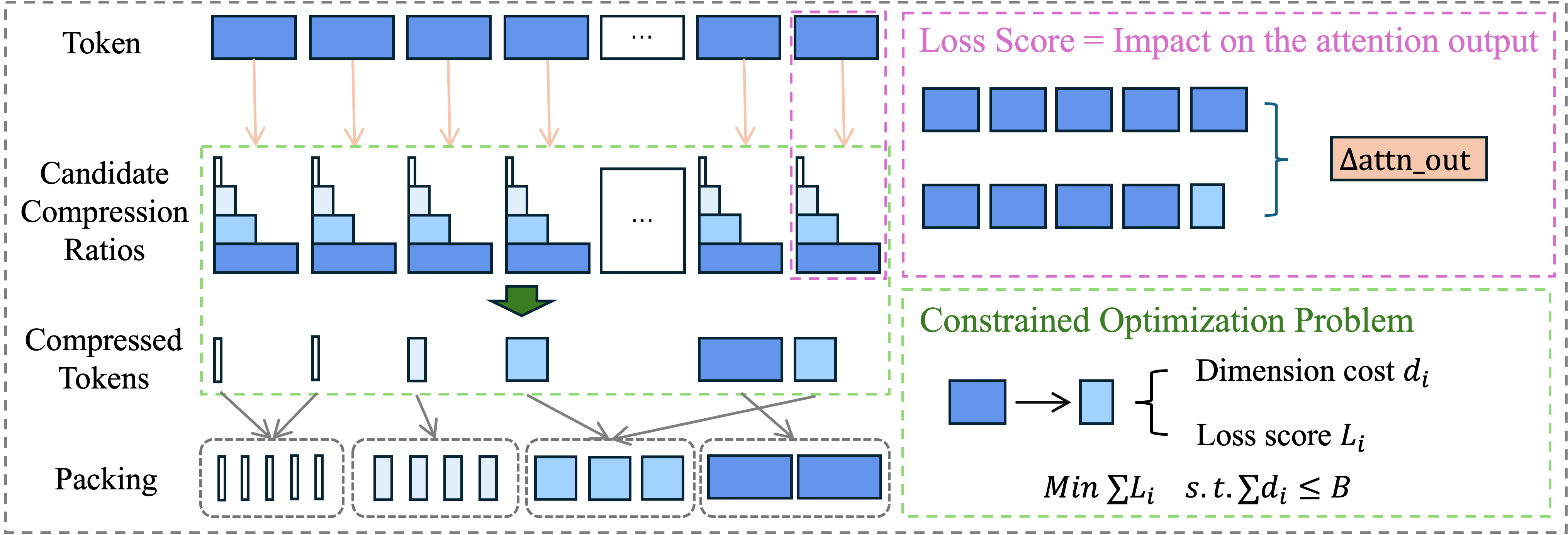}
    \caption{The two-stage dimension allocation process for KV cache compression. Candidate compression ratios are evaluated for each token to compute loss scores. An optimization objective is then applied to minimize total loss under the specified memory budget.}
    \label{fig:overview}
\end{figure*}

Support for long context is emerging as a core capability of large language models (LLMs) \cite{gemini-2.5, qwen3, glm-4.5}, enabling complex tasks like information retrieval from large-scale documents.
However, this capacity is often constrained by the KV cache, which stores the key-value states of the previous tokens to avoid re-computation.
As the sequence length grows, the KV cache size increases linearly, leading to a massive memory footprint and memory access overhead during the decoding stage, which ultimately results in poor inference efficiency.

Research community has explored tremendous methods for KV cache compression \cite{duoattention, magicpig, razorattention, vl-cache}. 
Token eviction \cite{h2o, streamingllm, snapkv} is one of the representative methods due to its simplicity, feasibility, and effectiveness.
The key idea is to maintain the ``important'' tokens in the cache and evict ``unimportant'' ones.
Recently, some works \cite{palu, eigen-attention, xkv} have explored the possibility of KV cache compression via dimensionality reduction, which compress the keys and values to a fixed dimension.
Our inspiration comes from such a perspective: in the token eviction scheme, the evicted tokens can be considered as compressed to 0 dimensions, and the maintained tokens can be considered as maintaining the full dimensions.
Therefore, the token eviction scheme can be regarded as a special case for dimensionality reduction scheme, where each token is compressed to either 0 dimension or full dimension.
Then a natural question arises:
\textit{Can we allocate dimensions to tokens at a more granular level to maximize information retention under a fixed budget?}

% We first investigate the low-rank properties of keys and values under different methods.
% Considering both the compression efficiency and computational efficiency during inference, we choose to conduct dimension reduction for keys with combined heads, and conduct dimension reduction for individual head of values separately.
% In order to assign dimensions for different tokens, we compute a loss score to evaluate the performance drop when the corresponding key and value of a token is compressed to specified dimensions.
% Then the dimension assignment problem is converted to a constrained optimization problem: 
% given a total dimension budget, how to allocate dimensions so that the total loss is minimized?
% Utilizing the monotonicity of the accuracy loss related to compression dimension, we use Lagrange multipliers for the budget constraint and apply a binary search algorithm to efficiently computing the optimal solution.

% Figure: show how dimension allocation means, real example
We propose \algname{}, a novel KV cache compression scheme that enables fine-grained budget allocation.
As illustrated in \cref{fig:overview}, our framework operates on a pre-defined set of candidate compression ratios, specifying the fractional budgets for keys and values.
Then our solution efficiently prunes the KV cache by assigning more dimensions to keys and values that are most sensitive to information loss.
To utilize every bit of budget space as efficiently as possible while considering inference computational efficiency, we have made the following designs:

\textbf{Measuring information loss via attention output.}
To determine the optimal dimensionality for each token, we establish a scoring metric that quantifies the deviation induced by compression. Our methodology is centered on the joint impact of a token's attention significance, value magnitude, and inherent compressibility across varying dimensions.
While traditional methods often rely on attention weights to select tokens, which only consider the impact of keys, our approach evaluates the attention output directly to capture how these three factors interact to determine actual information loss.
Based on this analysis, we design our metric as a conservative estimation, specifically the upper bound of the deviation of the attention output, to ensure maximum information retention for critical tokens. 
This formulation is further optimized for efficient batched computation, ensuring that our fine-grained dimension allocation remains scalable even for extremely long sequences.
% Crucially, while a naive computation of this deviation for $N$ tokens incurs an $O(N^2)$ complexity, we introduce a highly efficient approximation that reduces the cost to $O(N)$.
% By decomposing the global attention deviation into a localized impact score, we eliminate the need for redundant re-calculations.
% This ensures our fine-grained dimension allocation remains scalable even for extremely long sequences.

% We evaluate the importance by computing a loss score for each token compressed to each optional compression ratio.
% Previous works on token eviction usually evaluate the importance by the attention weights, e.g., higher attention weights means higher importance of tokens.
% By such method, only keys take effect and the influence of values is missing.
% We instead choose to evaluate the change of attention output after the dimension reduction.
% The change of attention output takes into account both the the keys and values, and the low-rank property of each token is also considered.
% Therefore, it is more suitable for our multi-dimensional compression scheme.

\textbf{Optimal dimension allocation via bisection search.}
Leveraging the computed loss scores, we formulate the dimension allocation task as a constrained optimization problem, aiming to minimize the aggregate accuracy loss under a memory budget. 
While such discrete problems are generally challenging, we show that in long-context settings, the duality gap between the primal problem and its Lagrangian dual becomes negligible. 
By exploiting the monotonic relationship between allocated dimensionality and accuracy loss, we employ an efficient bisection-based search to determine the optimal per-token dimensions.

% In cases where no head-level importance is given, we search across all heads for inter-head budget allocation. 
% When integrated with pre-defined head-level budgets, we propose \optalg{}, which performs intra-head refinement to further optimize the dimensionality within each head.

% K - head combine; V - head separate
% loss score
% dimension assignment

In implementing \algname{}, we favor head-wise compression over joint-head compression to strike a balance between allocation flexibility and computational efficiency. Furthermore, to ensure compatibility with established head-level importance estimation schemes, we propose a variant, \optalg{}.
On LongBench and RULER, \basicalg{} significantly outperforms prior methods without head-level profiling, while \optalg{} consistently exceeds HeadKV when using the same head-level information.
Notably, our approach achieves comparable performance to full attention on LongBench with only 6.25\% (KV size = $512$) of the KV cache.
In the Needle-in-a-Haystack test, our solution maintains 100\% accuracy at a 50K context length while using as little as 0.26\% (KV size = $128$) of the cache. Furthermore, our solution is highly efficient, reducing per-token decoding latency to 55\% of full attention.

%% file: Body/2.related.tex
\section{Related Works}

\textbf{Token eviction.}
A common KV cache compression strategy is to analyze token importance and evict the less important ones.
StreamingLLM \cite{streamingllm} marks the initial tokens and the tokens in the recent window as important.
H2O \cite{h2o} identifies important tokens via cumulative attention scores and dynamically updates KV cache.
SnapKV \cite{snapkv} scores the token importance based on the attention score with an observation window.
PyramidKV \cite{pyramidkv} ranks tokens in the lower layer as more important and allocates more budget.
HeadKV \cite{headkv} further evaluates the importance of individual heads and incorporates these estimates into KV cache allocation process.

\textbf{Dimension reduction.}
Several previous works have recognized the low-rank nature of KV cache and leveraged it for compression.
Palu \cite{palu} proposes a low-rank decomposition for the linear projection matrices, which naturally reduces the hidden dimension of KV cache.
Eigen Attention \cite{eigen-attention} conducts low-rank decomposition on KV cache and proposes a layer-wise rank allocation algorithm to decide each layer's compression ratio.
ShadowKV \cite{shadowkv} conducts low-rank decomposition for pre-RoPE keys, which requires reconstructing them during the decoding stage.
None of these works has considered assigning different ranks to tokens with varying levels of importance.

\textbf{Sparse attention.}
Some works store the entire KV cache but dynamically load a small portion of KV during the decoding stage.
Quest \cite{quest} approximates keys by the channel-wise minimal and maximal values and selects important pages of tokens by estimating attention weights.
SparQ \cite{sparq} selects important channels of keys to estimate attention weights, and then fetches keys and values with high contribution.
PQCache \cite{pqcache} applies Product Quantization techniques to retrieve top-k relevant key-value pairs.

\textbf{Quantization.}
Many works compress KV cache by reducing the bit width.
KVQuant \cite{kvquant} applies several techniques to achieve a 3-bit quantization.
PolarQuant \cite{polarquant} encodes KV caches with quantized radii and angles.
While these works achieve good performance, they are orthogonal to our method.

%% file: Body/3.preliminary.tex
\section{Preliminary}

\subsection{Attention Mechanism and KV Cache}

In transformer-based models, the multi-head attention (MHA) module with $H$ heads projects the input sequence of embeddings into $H$ sets of queries, keys, and values.
The attention output for each head $i$ is computed as:
$$Attn(Q_i, K_i, V_i) = softmax(\frac{Q_iK_i^\top}{\sqrt{d}})V_i$$

During the decoding stage, to avoid excessive recomputation for the keys and values of previous tokens, the key-value pairs are cached.
The memory overhead of KV cache scales linearly with the number of heads and sequence length, leading to prohibitive memory costs.

To mitigate this, grouped-query attention (GQA) \cite{gqa} partitions the $H$ query heads into $G$ groups. 
Within each group, multiple queries share one key and value head.
By reducing the number of KV heads from $H$ to $G$, the KV cache size is substantially reduced.
Many recent LLMs \cite{qwen3, mistral7b, llama3} have applied GQA architecture in their models.

\subsection{Principal Component Analysis}
We employ Principal Component Analysis (PCA) \cite{pca} algorithm to compress the KV cache.
Specifically, for a key matrix $K \in \mathbb{R}^{n\times d}$, PCA performs singular value decomposition on the covariance-like matrix $K^\top K / n$, yielding $K^\top K / n = U \Sigma U^\top$.
% Here $\Sigma$ is a diagonal matrix, with singular values arranged in descending order.
To achieve a $\rho$ compression ratio, we construct a projection matrix $P_K = U_r\in \mathbb{R}^{d\times r}$, where $U_r$ consists of the first $r = d \cdot \rho$ columns of $U$ corresponding to the largest eigenvalues.
The original key $K$ is projected to a lower-dimensional representation $K_c =K\cdot P_K \in \mathbb{R}^{n\times r}$.
Similarly, the value matrix $V$ is compressed into $V_c = VP_V$ using its respective projection matrix $P_V$.
Our method reduces the KV cache storage from $2nd$ to $2nr$, and also introduces a fixed overhead of $2dr$ to store the projection matrices.

During inference, we first project the query $q$ into the $r$-dimensional subspace: $\hat{q}=qP_K$. 
The attention weights are then efficiently computed in this reduced space.
After obtaining the weighted sum of compressed values $V_c$, we project the result back to the original $d$-dimensional space.
The complete process is formulated as
\[
Attn(q, K_c, V_c) =softmax(\frac{qP_KK_c^\top}{\sqrt{d}})V_c P_V^\top
\]

%% file: Body/4.solution.tex
\section{Solution}

%Our solution presents a multi-dimensional compression for KV cache.
% As shown in \cref{fig:overview}, our solution assigns a latent dimension to each token from a pre-defined set of candidate compression ratios.
% Ideally, the tokens with higher importance and lower potential for compression should be set to higher dimensions, while the total KV cache consumption should stay within the budget.
% We present a simple yet effective method to solve the dimension allocation problem.
% We compute an accuracy loss score for each token, according to the change in attention output when the token is compressed to a specific candidate dimension.
% Then we optimize the total accuracy loss while ensuring the KV cache budget constraint.
% In \cref{sec:alg:score}, we explain our estimation for the accuracy loss for each token compressed to each optional dimension.
% In \cref{sec:alg:allocation}, we discuss how the latent dimension assignment problem is converted to a constrained optimization problem, and elaborate an efficient algorithm to search for the optimal assignment.
To process the KV cache of a user prompt, we partition the prompt into a local window $Q$ consisting of the last $\alpha$ tokens and the remaining tokens.
The local window is fully retained (i.e., $100\%$ compression ratio), while the remaining tokens are assigned compression ratios from a predefined candidate set. 
To maximize performance under a KV cache budget, we prioritize higher dimensions for important, less compressible tokens. 
For dimension reduction, we apply PCA to each head independently rather than across all heads jointly, owing to the superior efficiency and flexibility (\cref{sec:alg:head-wise}).
We quantify the impact of compression by computing an accuracy loss score, defined as the deviation in attention output, for each token-dimension pair (\cref{sec:alg:score}). 
The dimension allocation is then formulated as a constrained optimization problem, which we solve using an efficient search algorithm to find the optimal assignment (\cref{sec:alg:allocation}).

\subsection{Head-wise Compression}
\label{sec:alg:head-wise}

\begin{table}
    \caption{Comparison of head-wise and joint-head compression. Head-wise compression reduces projection overhead and enhances inference efficiency. Its support for head-level budget allocation ensures competitive compression performance.}
  \label{tab:compare-head-compress}
  \setlength{\tabcolsep}{3pt}
  \begin{center}
      \begin{small}
        \begin{tabular}{lccc}
          \toprule
          \textbf{Granularity} & 
            \makecell{Projection Matrices \\ Overhead} & 
            \makecell{Computational \\ Efficiency} & 
            \makecell{Compression \\ Quality} \\
          \midrule

        \textbf{Joint-head} & $\times$ & $\times$ & \checkmark \\
        \textbf{Head-wise} & \checkmark & \checkmark & \checkmark \\
          
          \bottomrule
        \end{tabular}
      \end{small}
  \end{center}
\end{table}

Let $H_{kv}$ be the number of KV heads, $N$ be the number of tokens in the KV cache, and $D$ be the head dimension.
We denote $\mathrm{X} \in \mathbb{R}^{H_{kv}\times N\times D}$ as either the key or value cache.
There are two ways to conduct PCA for $\mathrm{X}$.
\begin{enumerate}[leftmargin=*]
    \item \textbf{Head-wise compression.} We make dimensional reduction for each head individually.
    On each head, we compress dimensions from $D$ to a specific ratio.

    \item \textbf{Joint-head compression.} Heads are concatenated to form $X_{joint} \in \mathbb{R}^{N\times (H_{kv}\cdot D)}$.
    We compress dimensions from $H_{kv}\cdot D$ to a specific ratio.
\end{enumerate}

Although prior work \cite{palu} has observed a better low-rank property for joint-head compression, it may not lead to overall performance improvement.
As summarized in \cref{tab:compare-head-compress}, head-wise compression consistently exceeds or matches joint-head compression across multiple aspects.
We justify this choice through the following three factors.

First, head-wise compression is more memory-efficient regarding projection matrices.
For a compression ratio $\rho$, head-wise compression brings an overhead of $H_{kv} D^2\rho$, while joint-head compression requires $(H_{kv}D)^2\rho$, which is $H_{kv}\times$ higher. 
For a GQA model ($H=32$, $H_{kv}=8$, $D=128$) at a compression ratio of $\rho=25\%$, the joint-head overhead is equivalent to storing $64$ full tokens per layer per head, which is prohibitive for constrained budget settings (e.g., 128 tokens).

% with the same compression ratio $\rho$, joint-head compression has much higher memory overhead for the projection matrix $P$.
% The memory consumption of head-wise projection matrix is $(\rho\cdot D)\cdot D\cdot H_{kv}$, while that of the joint-head projection matrices is $(\rho \cdot H_{kv}\cdot D)\cdot (H_{kv}\cdot D)$, which is $H_{kv}\times$ higher.
% When the KV budget is limited, joint-head compression can be worse than the head-wise compression due to the fixed overhead of projection matrices.
% For instance, suppose the target model has $32$ attention heads, which are grouped into sets of $4$ using GQA, with a head dimension of $128$.
% If one optional compression ratio is $25\%$, then for joint-head compression the overhead of projection matrix is equivalent to storing $128$ full tokens per head and per layer, while head-wise compression consumes $16$ tokens.
% Such large overhead renders the performance unsatisfactory under the small budget settings (e.g., $256$, $128$) commonly adopted in existing studies.

Second, joint-head compression brings higher computation complexity.
Since attention is a head-level operation, joint-head compression requires either reconstructing the full KV cache, which negates any efficiency gains, or operating on a joint subspace. 
For example, at $\rho=25\%$ with $8$ heads, the joint dimension becomes $2\times$ the original head dimension, increasing rather than decreasing the per-head computational burden.

Third, and most importantly, joint-head compression enforces uniform behavior across all heads, ignoring the non-uniform importance across heads.
Recent work \cite{headkv} has pointed out that allocating non-uniform budgets to different heads can yield better performance.
Head-wise compression naturally supports this spatial adaptivity, whereas the gains from joint-head low-rankness are often canceled out by the lack of head-level adaptivity, which is confirmed by experiments in \cref{sec:exp:ablation}.
% Actually, even with a larger budget where the projection overhead becomes marginal, joint-head compression does not show advantage in compression quality (\cref{sec:exp:ablation}).

\subsection{Accuracy Loss Score}
\label{sec:alg:score}

To assess the impact of compressing a specific token, we employ an analysis that isolates the error introduced by the $t^{th}$ token, assuming all other tokens remain unchanged.
% We estimate $\sum_{q\in Q} \|\Delta(p_{q} V)\|_2 $, where $p_{q}$ refers to the attention scores of query $q$, and $\Delta(p_{q} V)$ indicates the change of attention output for query $q$ caused by the compression error of the $t^{th}$ token.
We target at the change of attention output for queries in the local window $Q$ and formulate it as $\sum_{q\in Q} \|p'_{q,t}V_t'-p_{q}V\|_2 $, where $p_{q,t}$ refers to the attention scores for query $q$ with the changed $t^{th}$ token, and $V_t'$ refers to the values with changed $t^{th}$ entry.
Given that the local window provides only a partial sampling of the queries, we adopt a conservative score estimation to ensure that critical tokens retain their original information as much as possible.
We consider the upper bound 
$$\sum_{q\in Q} \|\Delta p_{q,t}  V\|_2 +  \|p'_{q,t} \Delta V\|_2$$
To simplify the computation and facilitate efficient batched processing, we consider all tokens to be compressed at a specific ratio, yielding attention score $p'_{q}$ for query $q$.
Due to the low-rankness of the key-values and the sparsity of attention scores, we regard the $p'_{q}$ to be a slight perturbation of $p'_{q,t}$.
Therefore, we define the scoring metric as 
$$L_t = \sum_{q\in Q} \|  (p'_{q}-p_q)^{(t)} \cdot  V^{(t)}\|_2 +  \|p_{q}^{(t)}\cdot  (V' - V)^{(t)}\|_2 $$

\newcommand{\mycomment}[1]{\hfill \textcolor{gray}{// #1}}

\begin{algorithm}[tb]
  \caption{Loss Score Computation}
  \label{alg:loss-score}
  \begin{algorithmic}
    \STATE {\bfseries Input:} Key $K \in \mathbb{R}^{N\times d}$, Value $V\in \mathbb{R}^{N\times d}$, Query $Q\in \mathbb{R}^{M\times d}$, compression ratio $\rho$, head dimension $d$
    \STATE {\bfseries Output:} Loss score vector $\mathbf{L} \in \mathbb{R}^N$
    
    \STATE $P \gets Softmax(\frac{QK^\top}{\sqrt{d}})$ \mycomment{$P \in \mathbb{R}^{M \times N}$} 
    
    \IF{$\rho == 0\%$}
        \STATE $E \gets 2 \cdot P \odot \|V\|_2 $ \\
    \ELSIF{$\rho == 100\%$}
        \STATE $E \gets \mathbf{0}_{M \times N}$ \\
    \ELSE
    \STATE $K', V' \gets Compressed(K, V, \rho)$ \\
    \STATE $P' \gets Softmax(\frac{QK'^\top}{\sqrt{d}})$ \\
    \STATE $E \gets |P' - P| \odot \|V\|_2 + P \odot \|V - V'\|_2$ \\
     \ENDIF
    \STATE $\mathbf{L} \gets \sum_{i=1}^M E[i, :]$ \mycomment{Sum across queries}
  \end{algorithmic}
\end{algorithm}

Then we can efficiently compute the scores for the entire input sequences in batches.
As shown in \cref{alg:loss-score}, we first compute the error matrix $E = |P' - P| \cdot \|V\|_2 + P \cdot \|V - V'\|_2$, where $E_{ij}$ estimates the impact of the $j$-th token on the $i$-th query's output.
Then we compute the loss score by summing across queries, i.e., $L = \sum_{i=1}^M E[i, :]$.
Our algorithm also accounts for two boundary cases where the compression ratio $\rho$ is either $100\%$ or $0\%$.
When $\rho = 100\%$, i.e., no compression is applied, we simply set the error matrix to zero.
When $\rho=0\%$, i.e., evicted tokens, we zero out both the attention weights and the value vectors, yielding an error matrix of $2 \cdot P \odot \lVert V \rVert_2$.
% A special case is to compress the token to $0$ dimension, e.g., evict tokens.
% In such case, we set the corresponding softmax-normalized attention weight to $0$, and the corresponding value entry to the zero vector.
% $$\hat{L} = 2\sum_{q\in Q}p_q \odot norm(V) $$
% Another special case is to make no compression, e.g., the compressed dimension is equal to the original head dimension.
% In such case, the loss score is set to zero vector.

\textbf{Generalization of Token Eviction.}
SnapKV \cite{snapkv} focuses on token eviction by retaining only those tokens with the highest attention weights within an observation window.
If we set the candidate compression ratios to $\{0\%, 100\%\}$, our solution also functions as token eviction.
In such a case, to minimize the total loss score, tokens with highest $\sum_{q\in Q} p_{q}^{(t)} \|V^{(t)}\|_2$ are maintained.
Compared to SnapKV, our accuracy loss score incorporates both attention importance and value magnitudes, offering a more nuanced sensitivity metric for cache management.
% As we discussed before, previous work on token eviction can be regarded as a special case of our multi-dimensional compression with candidate compression ratios $(0\%, 100\%)$.
% The representative work, SnapKV \cite{snapkv} uses an observation window for query, and maintain tokens with highest attention weights.
% SnapKV \cite{snapkv} uses an observation window of query entries to compute the attention weights for remaining tokens, and maintain tokens with highest attention weights. 
% In our solution, if we set candidate compression ratios $(0\%, 100\%)$, to minimize the total loss score, tokens with highest $p_{q,t} v_t$ are maintained, and the difference is that we consider both the attention weights and the value norm.
% Therefore, the design of our accuracy loss score can be regarded as an extension to previous token eviction works.

% We present the time complexity of computing loss score for one dimension $u$.
% Consider one layer of KV cache, where the sequence length is $L$, the number of KV heads is $H_{kv}$, and the head dimension is $D$.
% Suppose the query window size for $Q$ is $L_q$
% Computing the attention scores for compressed keys has the time complexity of $O(HL_qDu + HLDu+HL_qLu)$.

\subsection{Latent dimension allocation}
\label{sec:alg:allocation}

% \begin{algorithm}[tb]
%   \caption{Dimension Allocation}
%   \label{alg:alloc}
%   \begin{algorithmic}
%     \STATE {\bfseries Input:} Loss $L$, Dimension set $D$, Budget $B$
%     \STATE {\bfseries Output:} Allocation scheme $A$ 
    
%     \STATE Initialize $\lambda_{min} \gets  0, \lambda_{max} \gets 1e6$. \\
%     \WHILE{True}
%     \STATE $\lambda \gets (\lambda_{min} + \lambda_{max})/2$ \\
%     \STATE $score  \gets L + \lambda \cdot D$ \\
%     \STATE $A \gets argmin(score)$ \\
%     \STATE $C \gets ConsumedBudget(A)$ \\
%     \IF{$C == B$}
%         \STATE Return $A$ \\
%     \ELSIF{$C > B$}
%         \STATE $\lambda_{min} \gets \lambda$ \\
%     \ELSE
%         \STATE $\lambda_{max} \gets \lambda$ \\
%      \ENDIF
%     \ENDWHILE
%   \end{algorithmic}
% \end{algorithm}

With token-wise loss scores computed, we then select the optimal dimension for each token to minimize the total loss while adhering to the KV cache budget.
Let $\mathcal{D}$ denote the set of candidate dimensions, $d_i \in \mathcal{D}$ be the dimension allocated to the $i$-th token, $L_i(d_i)$ be the corresponding accuracy loss, and $B$ be the total budget. 
The allocation is formulated as a constrained discrete optimization problem:
\begin{equation}
    \min_{\{d_i \in \mathcal{D}\}_{i=1}^N} \sum_{i=1}^N L_i(d_i) \quad \text{s.t.} \quad \sum_{i=1}^N d_i \le B.
    \label{eq:allocation_primal}
\end{equation}

Although the problem is non-convex due to the discrete domain, the primal–dual gap is negligible, especially for long-context scenarios.
% By the Shapley-Folkman Lemma~\cite{starr1969quasi}, the relative duality gap between the primary problem and its Lagrangian dual diminishes at a rate of $O(1/N)$.

\begin{theorem}
    \label{theory:dual-gap}
    The gap between \cref{eq:allocation_primal} and its Lagrangian dual is bounded by a small constant $\Delta$ independent of the sequence length $N$.
\end{theorem}

We defer the proof and related discussion in \cref{app:dual-gap}.
Therefore, we consider its Lagrangian dual:

\begin{equation}
    \max_{\lambda \ge 0} \min_{\{d_i\}} \left( \sum_{i=1}^N L_i(d_i) + \lambda \left( \sum_{i=1}^N d_i - B \right) \right)
    \label{eq:allocation}
\end{equation}

For a fixed $\lambda$, the optimization decouples into $N$ independent sub-problems:
\begin{equation}
    d_i^*(\lambda) = \arg\min_{d \in \mathcal{D}} \big( L_i(d) + \lambda d \big)
\end{equation}
We adopt the natural assumption that $L_i(d)$ is monotonically decreasing with respect to $d$, as higher-dimensional representations typically preserve more information. 
Consequently, a higher $\lambda$ encourages selecting a smaller $d$, making the total budget consumption $C(\lambda) = \sum_i d_i^*(\lambda)$ a monotonically non-increasing step function of $\lambda$.
%We utilize the monotonicity to efficiently search for the optimal $\lambda^*$ via the bisection method.
This monotonicity allows us to efficiently find the optimal $\lambda^*$ via bisection search. 
In each iteration, we update $\lambda$ and determine the corresponding $\{d_i^*\}$ until the aggregate budget $C(\lambda)$ converges to the target $B$.

\subsection{Inter-head vs. Intra-head Optimization}

Our proposed framework is highly versatile and can operate under two distinct settings depending on the availability of external information.

If no external information on the importance of layers and heads is given, the basic \basicalg{} distributes the budget evenly across all layers.
The optimization problem in \cref{eq:allocation_primal} is formulated across all heads and tokens within a single layer. 
By solving this problem, the basic \basicalg{} inherently enables joint budget allocation at both the inter-head and inter-token levels.
If some heads have a higher impact on the attention output, they will be allocated a higher budget.

\algname{} is also compatible with existing methods that provide explicit importance of layers and heads.
HeadKV \cite{headkv} profiles head-level importance by contextual QA tasks, and partitions the total budget among heads of all layers according to their estimated saliency.
We can retain the head-level quotas established by HeadKV, while employing our mixed-dimension optimization scheme to refine the token-wise distribution within each head.
Under this setting, the optimization problem in \cref{eq:allocation} is formulated independently for each head, focusing on intra-head token selection.
We denote this variant as \optalg{}.

%% file: Body/5.implement.tex
\section{Implementation}
\label{sec:implement}

In this section, we discuss the practical implementation of our solution, focusing on memory layout and projection matrix management.

% We reorganize the compressed KV cache by packing tokens with identical dimensions into unified groups. 
% This ensures that each group forms a contiguous tensor, enabling efficient vectorized operations.
% During inference, attention is computed within each group and subsequently aggregated. 
% Although this rearranges the physical memory layout, the mathematical correctness is preserved as positional information (e.g., RoPE) is inherently encoded within the cached representations.

To optimize the memory layout and the computational efficiency, we reorganize the compressed KV cache by packing tokens with identical dimensions into contiguous tensors. 
This layout enables vectorized attention kernels to operate on each group independently, with the results aggregated subsequently.
To avoid excessive memory fragmentation and kernel launch overhead, we limit the number of candidate ratios.
Although this rearranges the physical memory, mathematical correctness is preserved as the positional information (e.g., RoPE) is already explicitly embedded within the representations.

% Another issue is the storage of projection matrices. 
% As we construct projection matrices by truncating the decomposed matrices of $K^\top K/N, V^\top V/N$, we can retain the projection matrix for the maximum candidate dimension (except the full dimension).
% Any projection matrix for lower dimension can be obtained by simply truncating the maximum projection matrices.
Regarding projection storage, we exploit the nested structure of our matrices. 
Since projection bases are derived from the eigenvectors of $K^\top K/N$ and $V^\top V/N$ through truncation, the basis for a lower dimension is a subset of the higher-dimension one. 
Consequently, we only store the maximum-rank projection matrix (for the largest $\rho < 100\%$). Any lower-rank matrix is obtained by simply slicing the maximum matrix, reducing the fixed memory footprint.

%% file: Body/6.experiment.tex
\section{Experiments}

\input{Exp/setup}
\input{Exp/main-results}

\input{Exp/niah}
\input{Exp/efficiency}

\input{Exp/ablation}

%% file: Exp/setup.tex
\subsection{Setup}

\textbf{Models and Datasets.} 
We conduct experiments on open-sourced state-of-the-art LLMs: Llama-3-8B-Instruct \cite{llama3}, the long-context finetuned variant Llama-3-8B-Instruct-Gradient-1048K \cite{llama3-8b-1048k}, and Mistral-7B-Instruct-v0.3 \cite{mistral7b}.
We compare our solution with baselines on representative benchmarks, including Longbench \cite{longbench-v1}, RULER \cite{ruler}, and Needle-in-a-Haystack \cite{niah}.
The experiments are conducted on NVIDIA A100 80GB GPUs.
% For llama-3-8B, we only use Llama-3-8B-Instruct-Gradient-1048K for the Needle-in-a-Haystack test to evaluate the performance of extremely long context ($>50K$).

\textbf{Baselines.}
We use the following baselines:
\begin{itemize}[nosep, leftmargin=*]
    % \item StreamingLLM (SLLM) \cite{streamingllm} maintains the initial tokens and the last $\alpha$ tokens.
    \item H2O \cite{h2o} applies an update policy to maintain a set of heavy-hitter tokens.
    % \item Palu \cite{palu} compresses KV cache by conducting low-rank decomposition for linear projection matrices.
    \item SnapKV \cite{snapkv} uses an observation window consisting of the last $\alpha$ tokens to compute attention scores of the remaining tokens. Tokens with a higher sum of attention scores are maintained in the KV cache.
    \item PyramidKV \cite{pyramidkv} allocates more budget to lower layers and less budget to higher layers, and within each layer it applies SnapKV for KV selection.
    \item HeadKV \cite{headkv} allocates budget to each head according to the head's importance, and for each head, it applies SnapKV for KV selection.
\end{itemize}

Both our solution and many of the baselines will sample $\alpha$ last tokens for further KV cache compression, and for fairness we set $\alpha$ the same for all solutions in one experiment.
PyramidKV involves a parameter $\beta$ to control the layer-wise budget allocation, and we set $\beta=20$ as recommended in its paper.
HeadKV involves a parameter $\beta$ to control the head-wise budget allocation, and we set $\beta=1.1$ for both HeadKV and \optalg{}, which is one of the recommended parameter settings in its paper.
When computing the KV budget consumption, the projection matrices are also counted in our solution.
We set the candidate compression ratios to $\{ 0\%, 12.5\%, 25\%, 100\%\}$, and we discuss this choice in \cref{app:param}.
% In \cref{appendix:exp:palu} we compare \algname{} with PALU, a SOTA solution with dimension reduction strategy.

Beyond the general comparisons above, we further evaluate \algname{} against Palu, the SOTA solution that focuses on KV cache dimension reduction. 
The detailed results are presented in \cref{appendix:exp:palu}.

%% file: Exp/main-results.tex
\subsection{Main Results on Long-Context Benchmarks}

We evaluate our solution on long-context benchmarks, including LongBench and RULER.

\begin{table*}[t]
  \caption{Performance comparison on LongBench. \basicalg{} (MD) is compared against KV cache compression methods that do not rely on additional head-level importance profiling.}
  \label{tab:longbench}
  \setlength{\tabcolsep}{1.5pt}
  \newcommand\rot[1]{\rotatebox[origin=c]{45}{#1}}
  \centering
    \footnotesize
        \begin{tabular}{lccccccccccccccccc}
          \toprule
          \multirow{2}{*}[-2.5ex]{Method} & \multicolumn{3}{c}{Single-Doc QA} & \multicolumn{3}{c}{Multi-Doc QA} & \multicolumn{3}{c}{Summarization} & \multicolumn{3}{c}{Few-shot Learning} & \multicolumn{2}{c}{Synthetic} & \multicolumn{2}{c}{Code} & \multirow{2}{*}[-2.5ex]{Avg.} \\
            \cmidrule(lr){2-4} \cmidrule(lr){5-7} \cmidrule(lr){8-10} \cmidrule(lr){11-13} \cmidrule(lr){14-15} \cmidrule(lr){16-17}
            % \cline{2-14}
          & \rot{NQA} & \rot{Qasper} & \rot{MF-en} & \rot{HQA} & \rot{2WQA} & \rot{Musq} & \rot{GRep} & \rot{QMSum} & \rot{MNews} & \rot{TREC} & \rot{TQA} & \rot{SAMSum} & \rot{PCount} & \rot{PRe} & \rot{Lcc} & \rot{RB-P} & \\
          
          \midrule
            \multicolumn{18}{c}{\small \rule{0pt}{8pt}\textit{Llama-3-8B-Instruct KV size = 128}} \\ 
        \midrule

        Full & 25.56 & 32.30 & 39.71 & 43.56 & 35.49 & 21.14 & 28.58 & 23.22 & 26.67 & 74.00 & 90.48 & 42.29 & 4.80 & 69.75 & 59.13 & 54.02 & 41.92 \\
        H2O & 23.52 & 13.74 & 25.26 & 29.07 & 22.72 & 14.67 & 21.94 & 21.34 & 24.01 & 46.50 & 87.99 & 35.12 & 5.36 & 68.21 & 54.77 & 48.73 & 33.93 \\
        SnapKV & 22.39 & 15.98 & 30.97 & 40.78 & 28.60 & 19.35 & 19.76 & 21.88 & 21.40 & 65.50 & 89.72 & 38.77 & 5.75 & 69.00 & 58.76 & 54.73 & 37.71 \\
        Pyramid & 21.40 & 16.92 & 31.62 & 38.45 & 28.72 & 18.59 & 19.97 & 22.48 & 20.96 & 66.50 & 89.35 & 38.39 & \textbf{5.92} & 69.00 & 57.86 & 51.80 & 37.37 \\
        % HeadKV & 22.56 & 28.04 & 39.64 & \textbf{44.40} & 31.83 & \textbf{20.62} & 22.16 & 23.04 & 24.04 & 71.00 & 90.10 & 38.81 & 4.89 & 68.75 & 61.72 & \textbf{61.77} & 40.84 \\
        MD & \textbf{25.22} & \textbf{29.07} & \textbf{38.30} & \textbf{43.20} & \textbf{33.84} & \textbf{19.58} & \textbf{24.53} & \textbf{23.12} & \textbf{26.03} & \textbf{72.00} & \textbf{90.61} & \textbf{42.16} & 5.54 & \textbf{69.50} & \textbf{59.32} & \textbf{55.16} & \textbf{41.07} \\
        % MD-H & 24.59 & 27.15 & \textbf{41.39} & 43.59 & \textbf{35.10} & 20.59 & 23.43 & 22.74 & 25.56 & 70.00 & 90.34 & 40.51 & 5.08 & \textbf{69.50} & \textbf{62.30} & 59.69 & \textbf{41.35} \\

        \midrule

        \multicolumn{18}{c}{\small \rule{0pt}{8pt}\textit{Llama-3-8B-Instruct KV size = 512}} \\ 
        \midrule

        Full & 25.56 & 32.30 & 39.71 & 43.56 & 35.49 & 21.14 & 28.58 & 23.22 & 26.67 & 74.00 & 90.48 & 42.29 & 4.80 & 69.75 & 59.13 & 54.02 & 41.92 \\
        H2O & 23.76 & 21.56 & 31.57 & 40.54 & 30.36 & 17.79 & 24.89 & 22.15 & 25.74 & 69.00 & 90.67 & 40.33 & \textbf{6.05} & 68.07 & 61.00 & 56.26 & 39.36 \\
        SnapKV & 25.53 & 23.59 & 38.58 & \textbf{43.78} & 33.33 & 19.85 & 23.10 & 22.52 & 24.19 & 71.00 & 90.57 & 40.48 & 5.51 & \textbf{69.50} & \textbf{61.10} & \textbf{57.23} & 40.62 \\
        Pyramid & 24.83 & 23.33 & 35.20 & 43.29 & 31.87 & 20.55 & 23.43 & 22.80 & 24.29 & 71.50 & \textbf{90.61} & 40.81 & 5.91 & \textbf{69.50} & 59.60 & 54.71 & 40.14 \\
        % HeadKV & 24.83 & 29.85 & 38.06 & \textbf{44.30} & 36.34 & \textbf{22.17} & 24.70 & 23.14 & 26.03 & \textbf{73.50} & 90.56 & 40.82 & 5.66 & \textbf{69.50} & \textbf{62.36} & \textbf{60.66} & 42.03 \\
        MD & \textbf{25.56} & \textbf{32.11} & \textbf{39.71} & 43.56 & \textbf{35.55} & \textbf{21.18} & \textbf{28.66} & \textbf{23.29} & \textbf{26.70} & \textbf{73.50} & 90.48 & \textbf{42.47} & 4.80 & 69.25 & 59.15 & 54.13 & \textbf{41.88} \\
        % MD-H & 25.37 & 31.96 & \textbf{39.82} & 43.97 & \textbf{36.91} & 21.39 & 27.61 & \textbf{23.35} & 26.24 & \textbf{73.50} & \textbf{90.64} & 41.83 & 5.27 & \textbf{69.50} & 62.02 & 58.29 & \textbf{42.35} \\

        \midrule

        \multicolumn{18}{c}{\small \rule{0pt}{8pt}\textit{Mistral-7B-Instruct KV size = 128}} \\ 

        \midrule

        Full & 29.07 & 41.54 & 52.88 & 49.37 & 39.01 & 28.58 & 35.07 & 25.71 & 27.73 & 76.00 &	88.59 & 47.51 & 6.00 & 98.50 & 61.48 & 62.68 & 48.11 \\
        H2O & 26.48 & 29.11 & 37.10 & 45.26 & 32.37 & 22.33 & 25.20 & 20.91 & 25.02 & 68.50 & 86.04 & 38.79 & 5.50 & 81.50 & 47.28 & 45.06 & 39.78 \\
        SnapKV & 26.78 & 30.63 & 48.03 & 47.58 & 35.09 & 25.35 & 21.90 & 22.12 & 21.83 & \textbf{69.50} & 88.59 & 43.91 & \textbf{6.00} & 94.00 & 55.70 & 55.14 & 43.26 \\
        Pyramid & 26.65 & 29.71 & 48.84 & 48.01 & 35.85 & 25.06 & 22.31 & 22.36 & 21.17 & 68.00 & 88.80 & 43.87 & 4.50 & 94.00 & 55.39 & 52.57 & 42.94 \\
        % HeadKV & 27.23 & 37.88 & \textbf{52.46} & 47.74 & \textbf{38.95} & \textbf{27.69} & 26.15 & 24.12 & 24.32 & \textbf{75.00} & \textbf{90.11} & 44.35 & 4.50 & 94.5 & 58.93 & 57.95 & 45.74 \\
        MD & \textbf{29.88} & \textbf{38.09} & \textbf{51.57} & \textbf{48.91} & \textbf{38.40} & \textbf{26.95} & \textbf{27.12} & \textbf{25.08} & \textbf{27.03} & \textbf{69.50} & \textbf{88.81} & \textbf{47.44} & 5.50 & \textbf{98.00} & 	\textbf{61.25} & \textbf{61.45} & \textbf{46.56} \\
        % MD-H & 28.22 & 38.08 & 51.12 & 48.79 & 37.75 & 27.32 & \textbf{27.38} & 24.73 & 26.63 & 66.50 & \textbf{90.11} & 44.95 & 5.23 & \textbf{98.00} & 61.14 & 60.38 & 46.02 \\

        \midrule
        \multicolumn{18}{c}{\small \rule{0pt}{8pt}\textit{Mistral-7B-Instruct KV size = 512}} \\ 

        \midrule

        Full & 29.07 & 41.54 & 52.88 & 49.37 & 39.01 & 28.58 & 35.07 & 25.71 & 27.73 & 76.00 &	88.59 & 47.51 & 6.00 & 98.50 & 61.48 & 62.68 & 48.11 \\
        H2O & 26.95 & 35.27 & 43.65 & 47.13 & 35.25 & 23.68 & 28.70 & 22.72 & 26.51 & 74.50 & 88.76 & 42.92 & \textbf{5.50} & 91.00 & 57.08 & 52.08 & 43.86 \\
        SnapKV & 29.22 & 36.49 & \textbf{54.05} & \textbf{49.70} & \textbf{38.72} & 26.72 & 26.01 & 24.24 & 25.26 & \textbf{75.00} & 89.44 & 46.85 & 5.00 & 96.00 & 60.27 & 60.66 & 46.48 \\
        Pyramid & 27.86 & 36.07 & 52.99 & 49.35 & 38.10 & 26.97 & 25.67 & 23.91 & 24.92 & 73.50 & \textbf{89.91} & 45.83 & 5.00 & 96.50 & 59.37 & 59.20 & 45.95 \\
        % HeadKV & 29.27 & 40.62 & 53.30 & 49.89 & \textbf{39.23} & \textbf{29.15} & 30.24 & 24.62 & 27.00 & \textbf{76.50} & \textbf{89.58} & 45.26 & 4.50 & 97.00 & 61.22 & 62.15 & 47.47 \\
        MD & \textbf{30.06} & \textbf{38.72} & 52.04 & 48.30 & 38.64 & \textbf{27.20} & \textbf{27.86} & \textbf{25.26} & \textbf{27.20} & 72.00 & 88.81 & \textbf{47.39} & \textbf{5.50} & \textbf{97.50} & \textbf{61.52} & \textbf{61.64} & \textbf{46.85} \\
        % MD-H & 29.91 & \textbf{40.87} & 53.17 & \textbf{50.09} & 39.14 & 28.32 & \textbf{32.43} & \textbf{25.77} & \textbf{27.50} & 75.50 & 89.06 & \textbf{47.60} & \textbf{6.00} & \textbf{99.00} & \textbf{61.67} & \textbf{62.60} & \textbf{48.04} \\
            
        \bottomrule
        \end{tabular}

  % \vskip -0.1in
\end{table*}

\begin{figure}[t]
  \centering
    \includegraphics[width=\columnwidth]{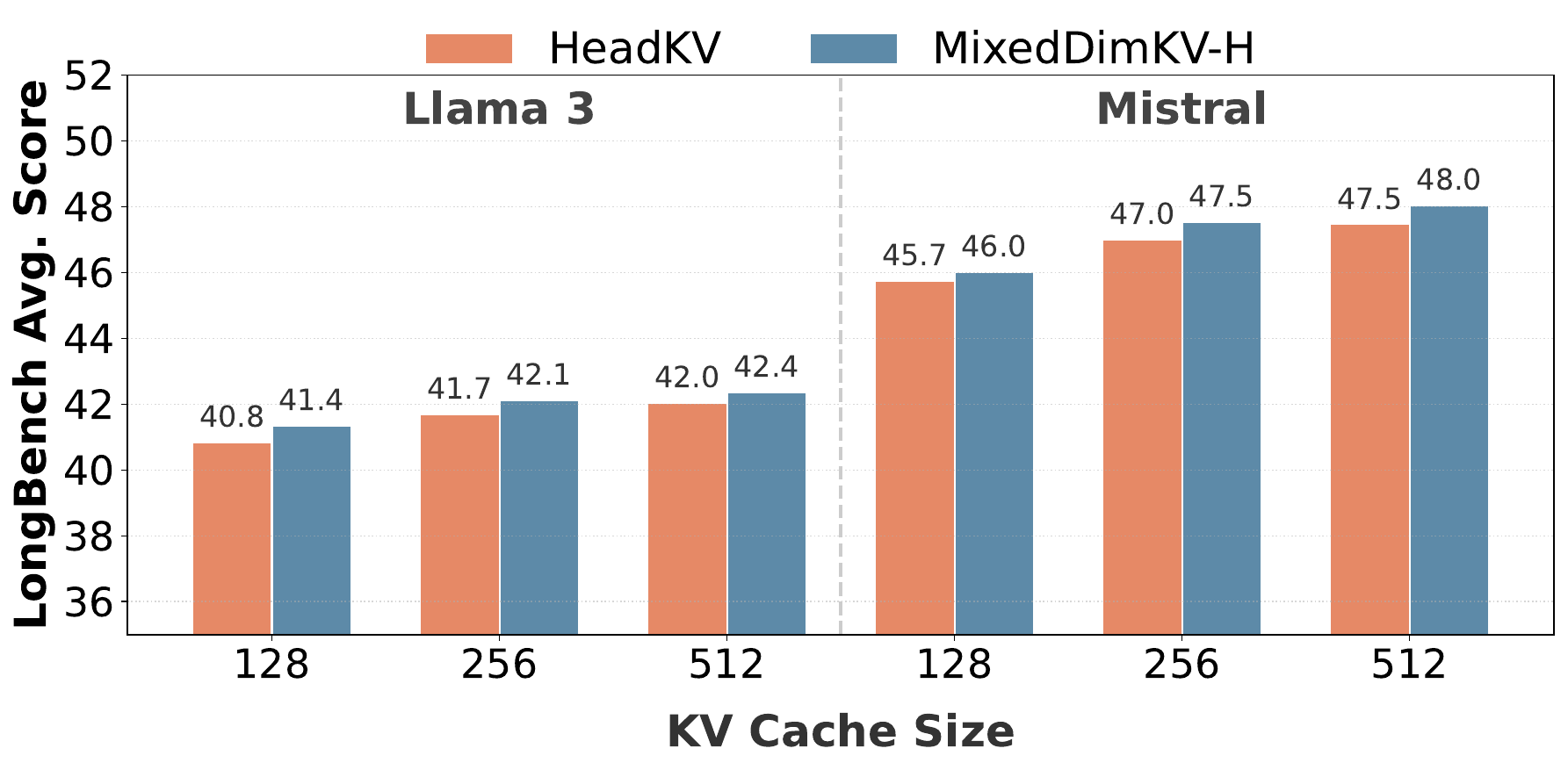}
    \caption{
      Comparison of HeadKV and \optalg{}.
    }
    \label{fig:exp:longbench:hkv-mdh}
\end{figure}

% \begin{table*}[t]
%   \caption{Performance of our solution and baselines on RULER.}
%   \label{tab:ruler}
%   % \setlength{\tabcolsep}{3pt}
%   \begin{center}
%     \begin{small}
%         \begin{tabular}{lcccccccccc}
%           \toprule
%           Method & S1 & S2 & MK1 & MK2 & MQ & MV & VT & QA-1 & QA-2 & Avg. \\

%           \midrule
%             \multicolumn{11}{c}{\small \rule{0pt}{8pt}\textit{Llama-3-8B-1048K}} \\ 
%         \midrule

%         Full & 100 & 100 & 99.4 & 99.40 & 99.25 & 83.40  & 96.64 & 69.83 & 58.40 & 89.59 \\
%         H2O & 1.6 & 0 & 0 & 0.8 & 0 & 0.05 & 8.24 & 23.23 & 23.4 & 6.37\\
%         SnapKV & 100 & 91.60 & 85.00 & 81.00 & 36.05 & 45.15 & 37.64 & 58.77 & 46.20 & 64.60\\
%         PyramidKV & 100 & 91.60 & 87.00 & 83.60 & 35.70 & 41.95 & 40.48 & 55.80 & 44.40 & 64.50\\
%         HeadKV & 100 & 99.60 &	99.20 & 96.00 & 95.65 & 86.05 & 83.88 & 64.20 & 53.20 & 86.42 \\
%         MD & 99.80 & 99.80 & 97.60 & 89.20 & 54.50 & 50.35 & 93.80 & 65.60 & 53.40 & 78.23\\
%         MD-H & 99.60 & 99.60 & 98.80 & 96.00 & 87.60 & 83.50 & 94.92 & 67.28 & 55.60 & 86.99\\

%         \midrule

%         % \multicolumn{11}{c}{\small \rule{0pt}{8pt}\textit{Mistral-7B-Instruct}} \\ 

%         % \midrule

%         % Full & \\
%         % H2O & \\
%         % SnapKV & \\
%         % PyramidKV & \\
%         % Palu & \\
%         % HeadKV & \\
%         % MD-SnapKV & \\
%         % MD-HeadKV & \\
            
%           \bottomrule
%         \end{tabular}
%     \end{small}
%   \end{center}
%   % \vskip -0.1in
% \end{table*}

\begin{figure}[t]
  \centering
    \begin{subfigure}{0.48\linewidth}
        \centering
        \includegraphics[width=\linewidth]{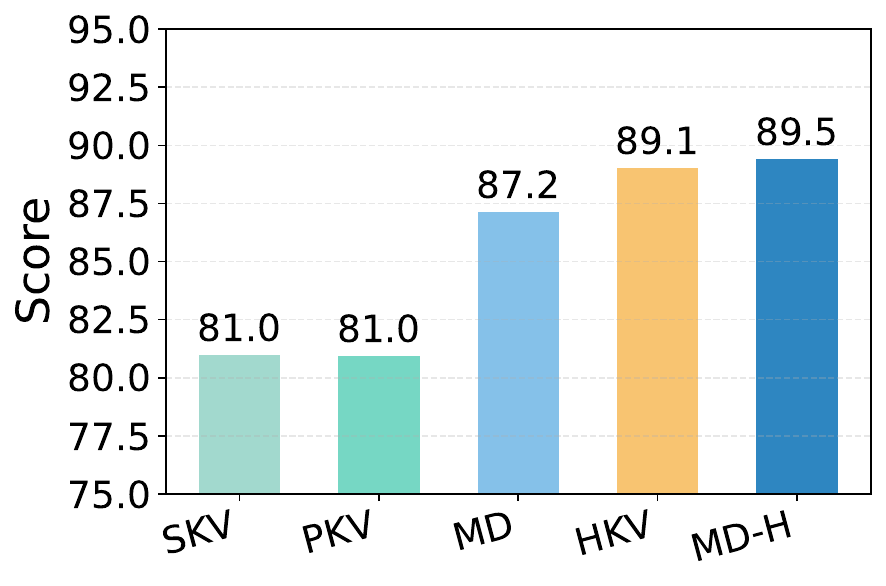}
        \caption{8K Context Length}
        \label{fig:exp:ruler:8k}
    \end{subfigure} \hfill
    \begin{subfigure}{0.48\linewidth}
        \centering
        \includegraphics[width=\linewidth]{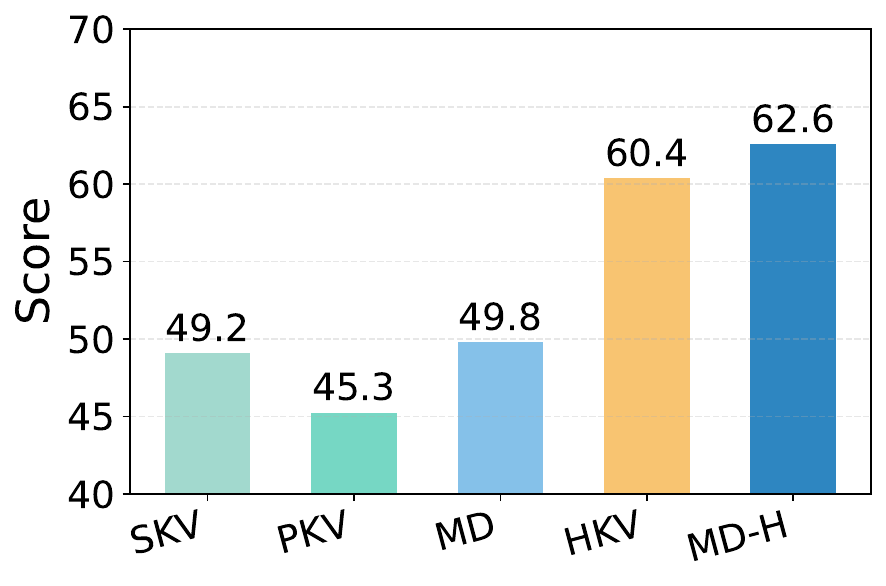}
        \caption{32K Context Length}
        \label{fig:exp:ruler:32k}
    \end{subfigure}
    \caption{
      Performance comparison on RULER. The terms `SKV', `PKV', `MD', `HKV', `MD-H' denote SnapKV, PyramidKV, \basicalg{}, HeadKV, \optalg{}, respectively.
    }
    \label{fig:exp:ruler}
\end{figure}

We evaluate performance on LongBench using equivalent KV sizes of 128, 256, and 512.
Following \cite{pyramidkv, headkv}, a KV size $T$ denotes a total budget of $H \times T \times D$, where $H$ is the number of heads and $D$ is the head dimension.
Since our approach introduces dimension reduction and projection overhead, this metric represents the equivalent memory budget rather than raw token counts. 
Due to space constraints, we present primary results here and provide the full results in \cref{appendix:exp:longbench}.

As shown in \cref{tab:longbench}, \basicalg{} significantly outperforms baselines lacking head-level importance profiling (H2O, SnapKV, PyramidKV), especially under tight budgets (KV size = 128).
At a KV size of 512, \basicalg{} matches full-attention performance on Llama-3. 
\cref{fig:exp:longbench:hkv-mdh} further shows that \optalg{} consistently exceeds HeadKV using identical head-importance profiling data, confirming the efficacy of mixed-dimension compression.
Notably, even without profiling, \basicalg{} outperforms HeadKV at 128 KV size on both Llama-3 (41.07 vs. 40.84) and Mistral (46.56 vs. 45.74).
We attribute this to HeadKV’s reliance on static, offline metrics derived from extensive Needle-in-a-Haystack tests. While these metrics capture general head importance, they fail to identify which heads are more critical within the current KV cache context. 
In contrast, our method dynamically adapts to the current input, enabling more precise head-level allocation when resources are limited.
% A more promising strategy would be to combine the strengths of both approaches, which we leave as future work.

\cref{fig:exp:ruler} presents the comparison of the average score on RULER with context lengths of 8K and 32K, using a fixed KV cache size of 256.
We omit H2O from the visualization due to its significantly lower scores, and the full results can be found in \cref{appendix:exp:ruler}.
The results on RULER exhibit a similar trend. \basicalg{} consistently surpasses SnapKV and PyramidKV in overall performance.
Notably, in the more challenging 32K setting, \optalg{} not only achieves the highest score but also broadens its lead over the competitive HeadKV baseline, reaching a peak performance of $62.64$.
This indicates our mixed-dimension compression enables more efficient memory usage when the budget is extremely limited.

%% file: Exp/niah.tex
\subsection{Needle-in-a-Haystack Test}

We evaluate the long-context information retrieval performance of different KV cache compression schemes using the Needle-in-a-Haystack test.
For Llama-3-8B-Instruct-1048k, we set the maximum context length to 50K, the KV size to $128$, and the query window size to $32$.
As shown in \cref{fig:exp:niah:llama}, our solution successfully retrieves important information under long-context scenarios.
\basicalg{} significantly outperforms the baselines that do not rely on additional head-level information (SnapKV, PyramidKV), and achieves a score close to HeadKV, which requires a time-consuming profiling process.
Equipped with the same head-level information, \optalg{} achieves a 100 score and outperforms HeadKV, indicating our mixed-dimension compression scheme is beneficial to the space utility of KV cache.
In \cref{appendix:exp:niah-mistral} we present the Needle-in-a-Haystack results on Mistral-7B-Instruct, which show the same trend as on Llama-3-8B-Instruct-1048K.
% For Mistral-7B-Instruct, we set the maximum context length to $32K$, which is the maximum supported context length of the model.

\begin{figure*}[t]
    \centering

    \begin{subfigure}{0.50\linewidth}
        \centering
        \includegraphics[width=\linewidth]{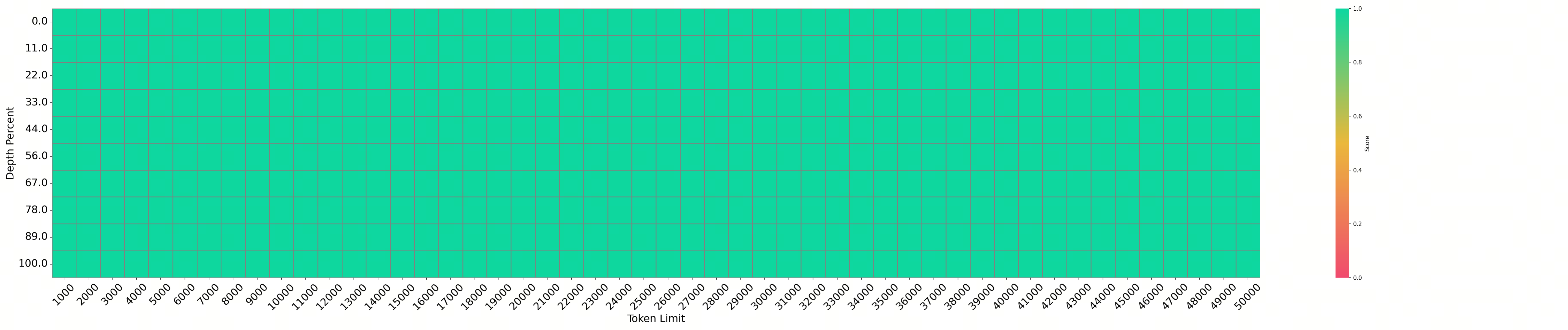}
        \caption{FullKV Score: 100}
    \end{subfigure}%\hfill
    \begin{subfigure}{0.50\linewidth}
        \centering
        \includegraphics[width=\linewidth]{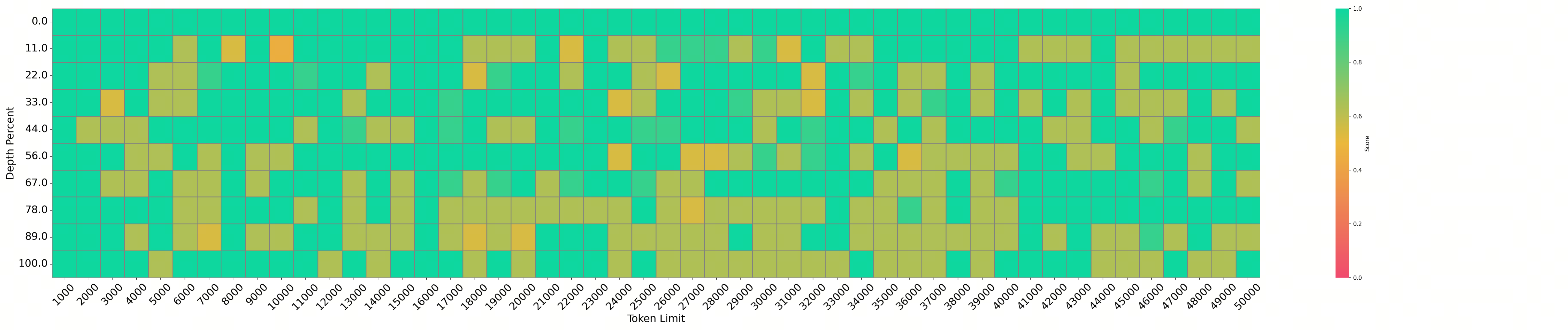}
        \caption{SnapKV Score: 85.80}
    \end{subfigure}

    % \vspace{0.5em}

    % Row 2
    \begin{subfigure}{0.50\linewidth}
        \centering
        \includegraphics[width=\linewidth]{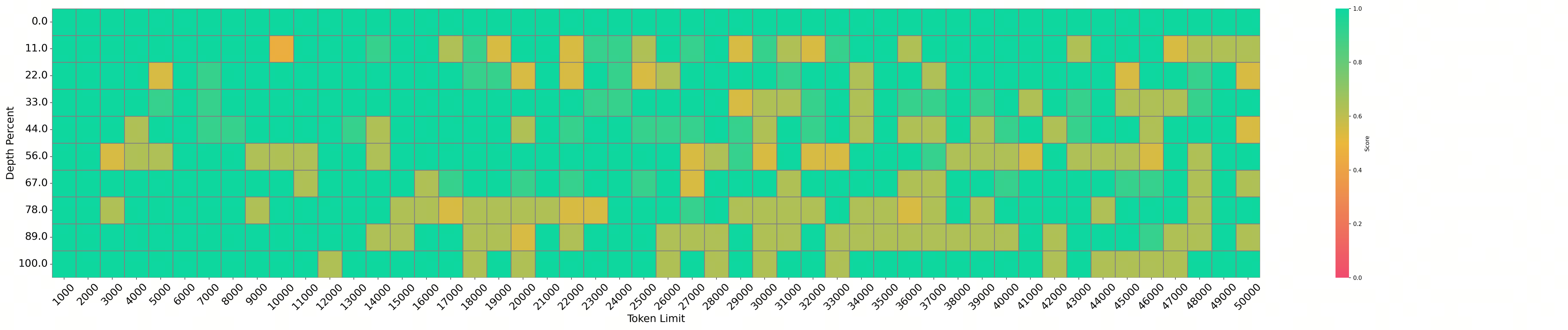}
        \caption{PyramidKV Score: 89.40}
    \end{subfigure}%\hfill
    \begin{subfigure}{0.50\linewidth}
        \centering
        \includegraphics[width=\linewidth]{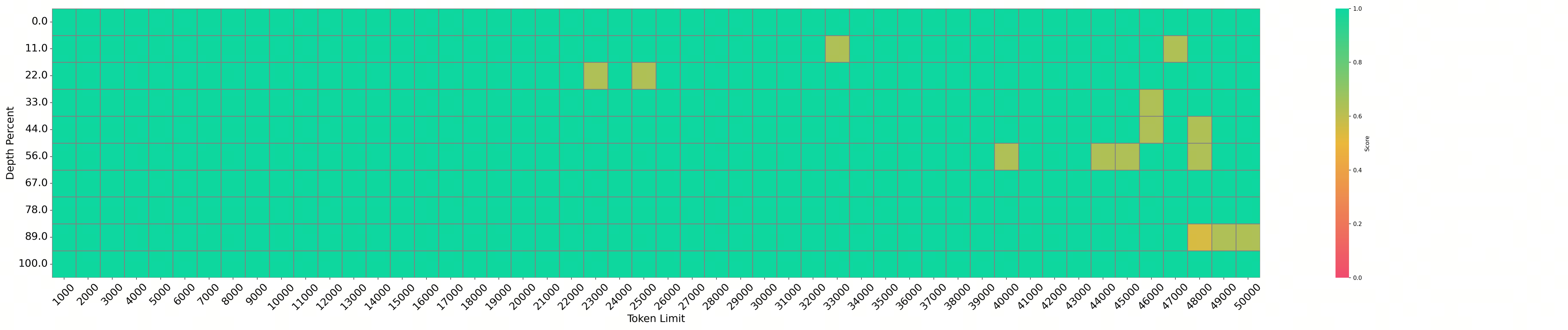}
        \caption{HeadKV Score: 99.00}
    \end{subfigure}

    % \vspace{0.5em}

    % Row 3
    \begin{subfigure}{0.50\linewidth}
        \centering
        \includegraphics[width=\linewidth]{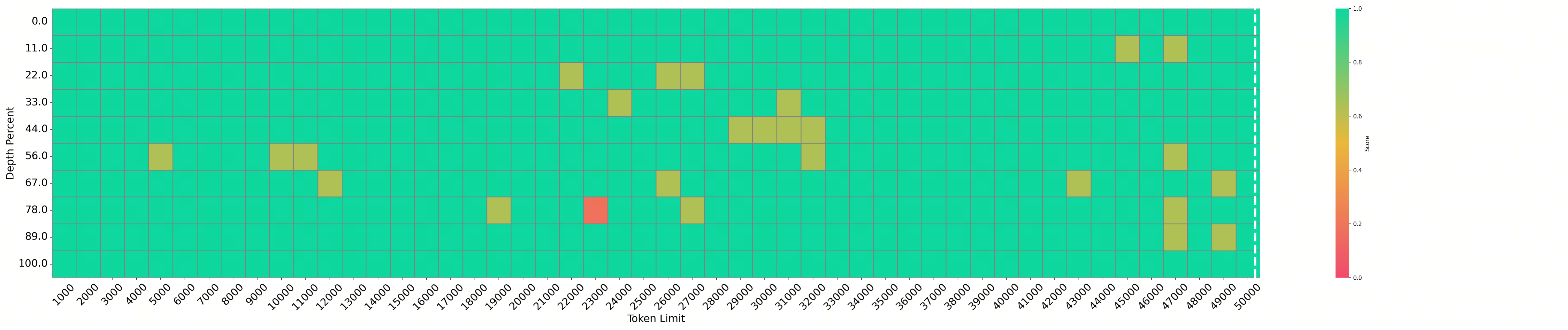}
        \caption{\basicalg{} Score: 98.00}
    \end{subfigure}%\hfill
    \begin{subfigure}{0.50\linewidth}
        \centering
        \includegraphics[width=\linewidth]{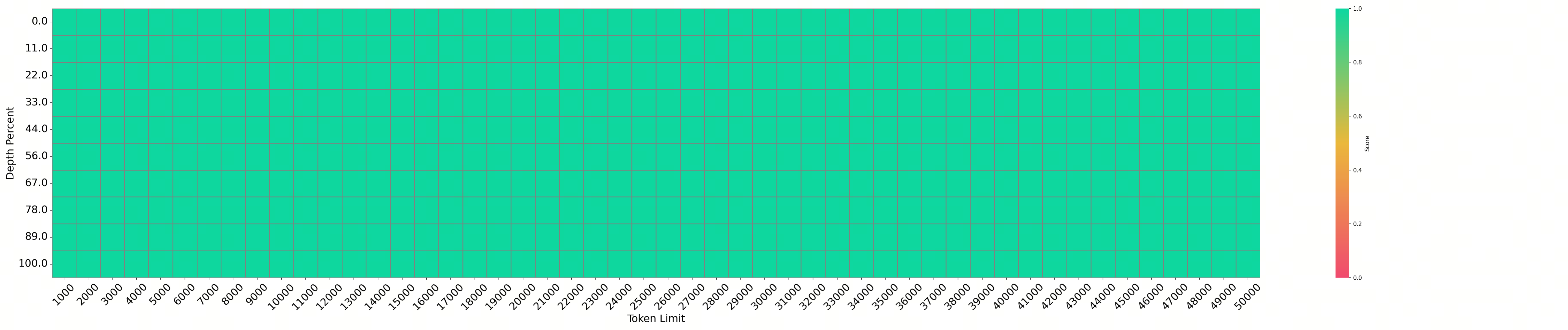}
        \caption{\optalg{} Score: 100.00}
    \end{subfigure}

    \caption{Results of Needle-in-a-Haystack test on Llama-3-8B-Instruct-1048K with KV size $128$.}
    \label{fig:exp:niah:llama}
\end{figure*}

%% file: Exp/efficiency.tex
\subsection{Inference Efficiency}

\begin{figure}[ht]
  \centering
    \includegraphics[width=0.45\linewidth]{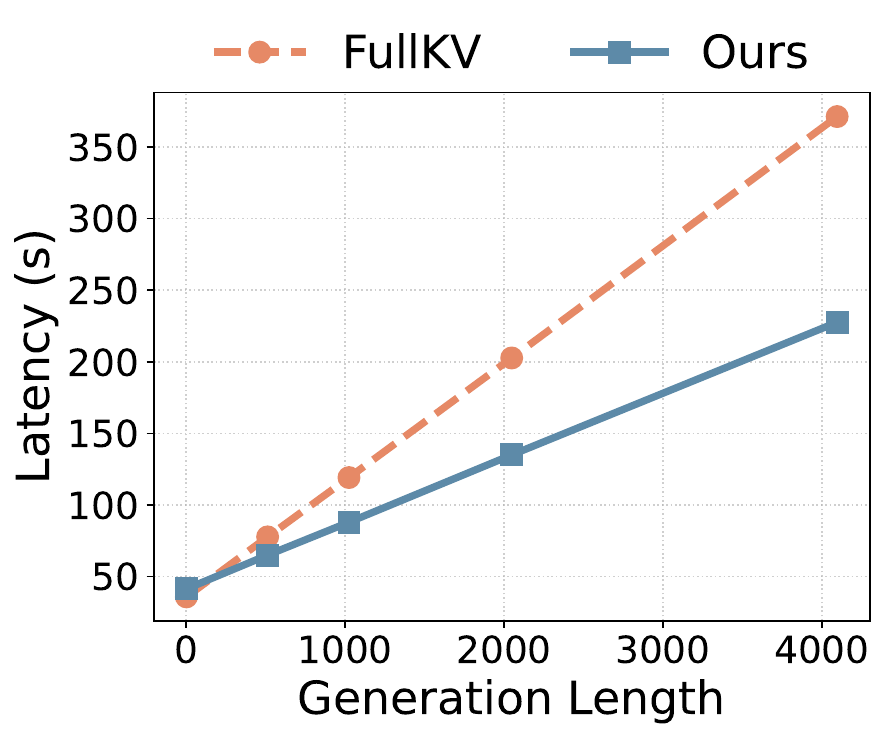}
    \includegraphics[width=0.45\linewidth]{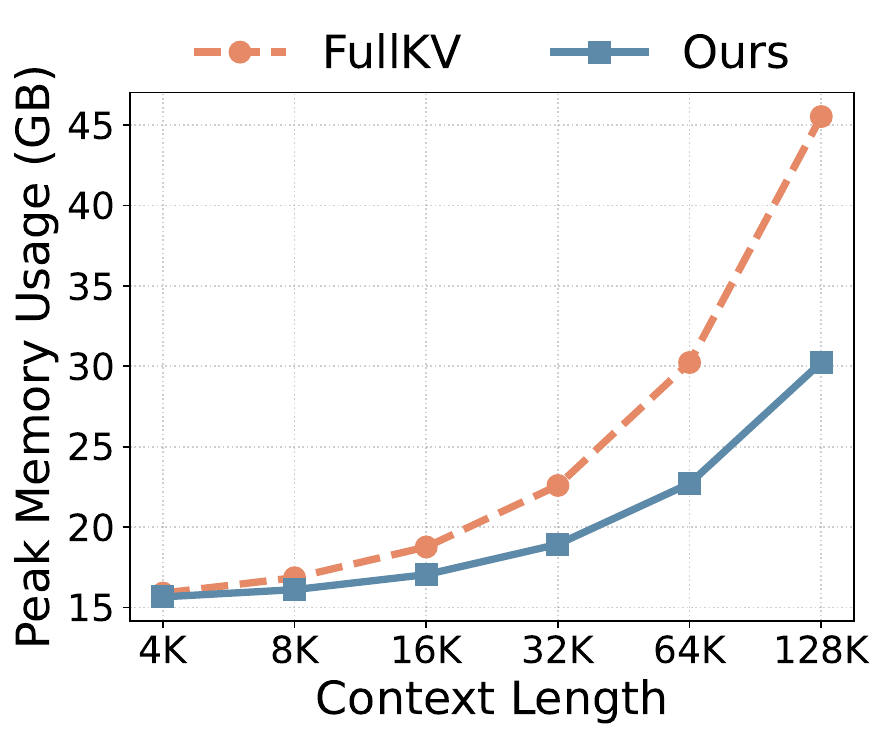}
    \caption{
      Decoding latency and peak memory usage.
    }
    \label{fig:exp:infer}
\end{figure}

In this section, we implement our solution in FlashAttention \cite{flashattention-2} and evaluate the inference efficiency on Llama-3-8B-Instruct-1048K.
% Experiments are conducted on Llama-3-8B-Instruct-1048K, and we construct a $128K$-length input from the Needle-in-a-Haystack dataset. 
% \cref{fig:exp:infer} shows the results of the decoding latency across different generation lengths and peak memory usage across different context lengths.
%The evaluation results of the decoding latency under different generation lengths and peak memory usage under different context lengths are shown in \cref{fig:exp:infer}.
\Cref{fig:exp:infer} reports decoding latency for $128K$ inputs across varying generation lengths, and peak memory usage across different context lengths with generation length set to $1$. 
We construct these inputs from the Needle-in-a-Haystack dataset.
When the generation length is $1$, the decoding latency represents the pre-filling latency, and our solution introduces minimal overhead, with latency only slightly higher than full attention (around $5$s for $128$K input).
For a generation length of $4096$, our solution reduces total and per-token latency to $61.31\%$ and $55.18\%$ of the baselines, respectively.
With $128K$ contexts, our solution lowers the peak memory usage to $66.44\%$.

%% file: Exp/ablation.tex
\subsection{Ablation Study}
\label{sec:exp:ablation}

In this section, we further conduct extensive ablation studies.
By comparing head-wise and joint-head compression schemes, we show that joint-head compression does not outperform head-wise compression.
We also evaluate our solution under different numbers of candidate dimensions, indicating that a more fine-grained compression leads to better performance.
In \cref{appendix:exp:dim-alloc} we report the number of tokens compressed to each dimension to guarantee it does not degenerate to the token eviction scheme.
The following experiments are conducted on the LongBench with Llama-3-8B-Instruct.

\begin{figure}[ht]
  \centering
    \includegraphics[width=\columnwidth]{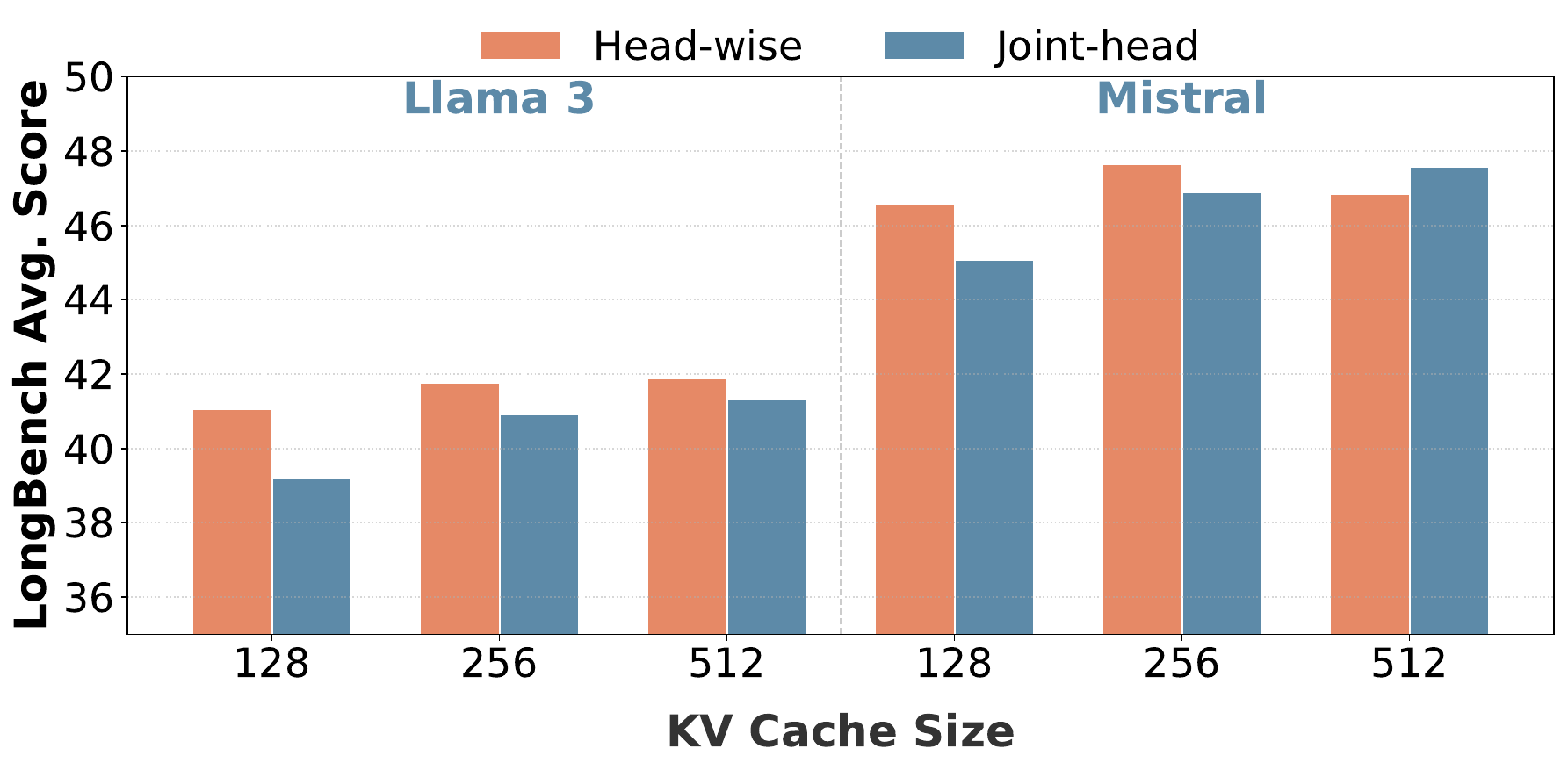}
    \caption{
      Comparison of head-wise and joint-head compression.
    }
    \label{fig:ablation:head}
\end{figure}

\textbf{Head-wise vs. joint-head compression.}
We conducted an ablation study by replacing head-wise compression with joint-head compression on keys to evaluate its impact.
Because heads are combined for compression, the latent dimension allocation process only searches for an optimal token-level budget distribution.
Both strategies share the same set of candidate dimension compression ratios and use no extra information of head-level importance.
As shown in \cref{fig:ablation:head}, head-wise compression achieves much better performance under a low budget (e.g., $128$ KV size), due to its reduced projection matrices overhead.
Even with a higher budget size, the joint-head compression does not show an advantage.
We believe this is because head-wise compression can distribute budget more flexibly as compression is conducted at the head-level.

% \begin{figure}[t]
%     \centering

%     \begin{subfigure}{0.48\linewidth}
%         \centering
%         \includegraphics[width=\linewidth]{Figure/Exp/granularity.pdf}
%         \caption{Compression granularity}
%         \label{fig:ablation:granularity}
%     \end{subfigure} \hfill
%     \begin{subfigure}{0.48\linewidth}
%         \centering
%         \includegraphics[width=\linewidth]{Figure/Exp/dim_num.pdf}
%         \caption{Dimension allocation ratio}
%         \label{fig:ablation:dim-num}
%     \end{subfigure}

%     \caption{\textbf{Left}: Config A, B refer to  different granularity of compression ratios, $\{0\%, 100\%\}$ and $\{0\%, 12.5\%, 25\%, 100\%\}$. \textbf{Right}: Average number of tokens and the allocation statistics across candidate compression ratios on five datasets of LongBench.}
%     \label{fig:ablation:granularity-dim-num}
% \end{figure}

\begin{table}
    \caption{Comparison of SnapKV and \basicalg{} with two granularity of candidate ratio sets: $\{0\%, 100\%\}$, and $\{0\%, 12.5\%, 25\%, 100\%\}$. }
  \label{tab:ablation:granularity}
  \setlength{\tabcolsep}{3pt}
  \begin{center}
      \begin{small}
        \begin{tabular}{lccc}
          \toprule
            Method & SnapKV & 
            $\{0\%, 100\%\}$ & 
            $\{0\%, 12.5\%, 25\%, 100\%\}$ \\
          \midrule

        Score & 39.77 & 40.93 & 41.76 \\
          
          \bottomrule
        \end{tabular}
      \end{small}
  \end{center}
\end{table}

\textbf{Granularity of mixed-dimension ratios.}
We investigate the influence of dimension granularity by evaluating two candidate ratio sets: $\{0\%, 100\%\}$ and $\{0\%, 12.5\%, 25\%, 100\%\}$, which represent mixed-dimension compression at varying levels of granularity.
We set the KV size to 256, and use SnapKV as the baseline. 
As shown in \cref{tab:ablation:granularity}, $\{0\%, 100\%\}$ outperforms the SnapKV baseline, and we attribute the superior performance to two advantages.
First, our solution adopts a more rational score design that incorporates value information, while SnapKV only considers the impact of keys.
Second, our dimension allocation mechanism inherently supports distinct budget ratios for different attention heads, in contrast to SnapKV, which enforces identical budget consumption across all heads.
Moreover, the performance leap from $\{0\%, 100\%\}$ to $\{0\%, 12.5\%, 25\%, 100\%\}$ indicates that progressively finer-grained dimension ratio design leads to better results.

% \paragraph{Dimension allocation statistics.}

% For most experiments, the candidate set of dimension compression ratios is set to $\{0\%, 12.5\%, 25\%, 100\%\}$.
% We report the average token length and the proportion of tokens allocated to each compression ratio, averaged over all layers and samples in each dataset.
% \cref{fig:ablation:dim-num} presents the results on five datasets of LongBench with KV size $256$, and more results can be found in \cref{appendix:exp:dim-alloc}.
% % and the results on other datasets can be found in \todo{}.
% We find that all four compression ratios take effect by allocating a portion of tokens.
% Specifically, tokens are predominantly allocated to the $0\%$ and $100\%$ compression ratios, while the two intermediate ratios account for a non-negligible 
% $10\%$-$15\%$ of the total tokens, which improves the space utilization of KV cache.
% Notably, in longer datasets, an even larger proportion of tokens are assigned to the $0\%$ ratio, which is a direct consequence of the constrained budget setting.

%% file: Body/7.conclusion.tex
\section{Conclusion}
In this paper, we present \algname{}, an efficient KV cache compression scheme.
\algname{} maximizes the budget utilization by compressing tokens to suitable dimensions based on their impact on the attention output.
% According to whether additional head-level importance information is utilized, \todo{} yields two variants:
When information of head-level importance is available, we further propose \optalg{}, which combines the benefits of mixed-dimension compression and head-level budget allocation.
Experiments on long-context benchmarks show that \basicalg{} significantly outperforms prior methods without head-level profiling, while \optalg{} consistently exceeds HeadKV.
On an A100, \algname{} reduces the per-token decoding latency to $55\%$ compared with full attention.
These results demonstrate that \algname{} significantly alleviates the memory bottleneck of LLMs, thereby extending the scalability of LLMs to long-context tasks within limited hardware budgets.

%% file: Body/appendix.tex
\input{Appendix/dual-gap}

\section{Parameter Choice}
\label{app:param}

\input{Appendix/parameter}

\section{Extended Experimental Results}

In this section, we provide additional experimental results to complement the analysis in the main text.
In \cref{appendix:exp:longbench} and \cref{appendix:exp:ruler}, we present additional results on LongBench and RULER, respectively.
In \cref{appendix:exp:niah-mistral} we show the results of Needle-in-a-Haystack tests on the Mistral model.

\input{Appendix/longbench}

\input{Appendix/ruler}
\input{Appendix/palu}
\input{Appendix/niah}

\input{Appendix/allocation}

%% file: Appendix/dual-gap.tex
\section{Proof of Theorem \ref{theory:dual-gap} and Gap Analysis}
\label{app:dual-gap}

In this section, we provide the formal proof for Theorem \ref{theory:dual-gap} and analyze the duality gap. We first prove a sequence-length-independent upper bound, then show that the relative gap vanishes as the context length grows, and finally explain why the realized gap at the converged dual variable is typically much tighter than the worst-case bound.

\subsection{Proof of Theoretical Bound}

Recall the primal problem $P^*$ in Eq. \eqref{eq:allocation_primal} and its dual $D^*$ in Eq. \eqref{eq:allocation}.

\begin{proof}[\textbf{Proof of Theorem \ref{theory:dual-gap}}]
    The problem consists of $N$ separable terms subject to $m=1$ linear constraint. The \textbf{Shapley-Folkman-Starr Theorem}~\cite{starr1969quasi} implies that the duality gap of a separable problem with $m$ linear constraints can be bounded by the sum of non-convexities of at most $m$ terms.
    
    Since we have exactly one constraint ($m=1$), the gap is bounded by the non-convexity of a single term:
    \begin{equation}
        P^* - D^* \le \max_{i=1,\dots,N} \rho(L_i) \eqqcolon \Delta.
    \end{equation}
    Here, $\rho(L_i)$ denotes the maximum deviation between $L_i$ and its convex envelope on the convexified domain $\mathrm{conv}(\mathcal D)$. Crucially, this bound $\Delta$ depends only on the properties of a single token's loss function. Therefore, $\Delta$ is a \textbf{constant} that remains \textbf{independent of the sequence length $N$}.
\end{proof}

\subsection{Vanishing Relative Error}

This analysis suggests that the relative impact of the duality gap diminishes as the context length increases.
Since the primal objective $P^*(N)$ is the summation of losses over $N$ tokens, the total value naturally grows with the sequence length (i.e., $P^* \approx O(N)$).

When comparing the bounded absolute gap against this growing total loss, the relative error tends to become negligible:
\begin{equation}
    \frac{\Gamma}{P^*(N)} \approx O(\frac{1}{N}) \to 0 \quad \text{as } N \to \infty.
\end{equation}
This indicates that for long-context scenarios, the solution provided by the Lagrangian relaxation is asymptotically near-optimal.

\subsection{Gap Tightness at Converged $\lambda^*$}

While the theorem guarantees that the gap does not scale with $N$, the theoretical constant $\Delta$ is a worst-case upper bound. Theoretically, the non-convexity $\rho$ could be as large as the total loss of a significant token ("loss jump"), which might suggest a potentially loose bound.

However, in practice, the duality gap is determined by the specific dual variable $\lambda^*$ to which our algorithm converges. Let $\hat d$ be the final \emph{feasible} allocation returned by our algorithm and let $D(\lambda)$ be the dual objective. By weak duality, for any $\lambda\ge0$ we have $D(\lambda)\le D^*\le P^*\le P(\hat d)$, hence
\begin{equation}
\Gamma = P^* - D^* \le P(\hat d)-D(\lambda^*) \eqqcolon \overline{\Gamma}(\lambda^*).
\end{equation}
In Section~\ref{app:dual-gap:empirical}, we verify that the realized bound $\overline{\Gamma}(\lambda^*)$ is typically much smaller than the worst-case $\Delta$.

\subsection{Empirical Verification}
\label{app:dual-gap:empirical}

We validate this analysis on the \textbf{MultiFieldQA-en} dataset from the \textbf{LongBench} benchmark. We sample sequences of varying lengths and evaluate the duality gap using \textbf{Llama-3-8B-Instruct}.

Table \ref{tab:gap_scaling} presents the results:
\begin{itemize}
    \item \textbf{Primal Loss ($P(\hat d)$):} As expected, the total objective value grows with context length (from 63.83 to 104.59).
    \item \textbf{Gap Upper Bound ($\overline{\Gamma}(\lambda^*)$):} The upper bound of duality gap $P(\hat d)-D(\lambda^*)$ remains small and shows no systematic growth with $N$.
\end{itemize}
Consequently, the relative gap drops to below $0.15\%$ for longer sequences, supporting the effectiveness of our approach for long-context tasks.

In addition to the sequence-level averages in Table~\ref{tab:gap_scaling}, Table~\ref{tab:layer_stats} shows that this tightness holds \emph{across layers} as well, indicating it remains small for the vast majority of layers.

\begin{table}[h]
    \centering
    \caption{Scalability of Duality Gap (MultiFieldQA-en from LongBench).}
    \label{tab:gap_scaling}
    \begin{tabular}{lccc}
        \toprule
        \textbf{Context Length ($N$)} & \textbf{Avg. Primal Loss ($P(\hat d)$)} & 
        \textbf{Avg. Gap Upper Bound ($\overline{\Gamma}(\lambda^*)$)} & \textbf{Avg. Relative Gap (\%)} \\
        \midrule
        2,005 & 63.83 & 0.023 & 0.033\% \\
        4,752 & 71.26 & 0.010 & 0.015\% \\
        6,186 & 104.59 & 0.010 & 0.011\% \\
        % 7,942 & 124.97 & 0.200 & 0.145\% \\
        \bottomrule
    \end{tabular}
\end{table}

\begin{table}[h]
    \centering
    \caption{Layer-wise Statistics (at $N=6,186$).}
    \label{tab:layer_stats}
    \begin{tabular}{lccccc}
        \toprule
        \textbf{Metric} & \textbf{Min} & \textbf{25th Percentile} & \textbf{Median} & \textbf{75th Percentile} & \textbf{Max} \\
        \midrule
        \textbf{Sensitivity ($\lambda^*$)} & $1.3 \times 10^{-5}$ & $9.3 \times 10^{-5}$ & $1.7 \times 10^{-4}$ & $2.2 \times 10^{-4}$ & $4.2 \times 10^{-4}$ \\
        \textbf{Relative Gap (\%)} & 0.000\% & 0.000\% & 0.000\% & 0.010\% & 0.085\% \\
        \bottomrule
    \end{tabular}
\end{table}

%% file: Appendix/parameter.tex
\begin{figure}
    \centering
    \includegraphics[width=0.5\linewidth]{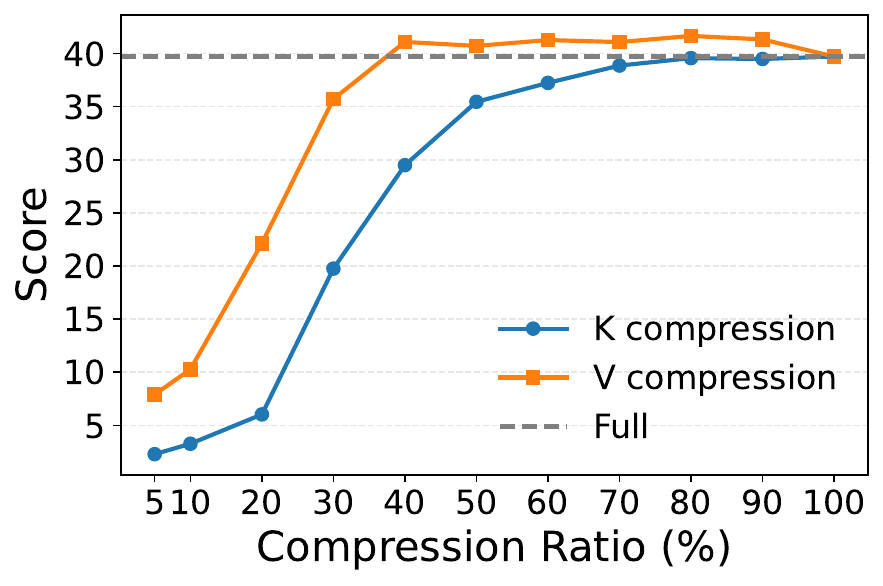}
    \caption{Performance of MultiFieldQA-en under different KV compression ratios. }
    \label{fig:kv-compression-curve}
\end{figure}

For most experiments of our solution, we set the candidate set of compression ratios to $\{0\%, 12.5\%, 25\%, 100\%\}$.
We evaluate the performance of KV dimensional reduction for Llama-3-8B-Instruct in one of the LongBench datasets (MultiFieldQA-en), where we vary the compression ratio for keys or values and report the scores.
As shown in \cref{fig:kv-compression-curve}, we observe a performance drop when the compression ratio is lower than $50\%$.
When selecting the optional compression ratios, the $0\%$ is selected to achieve low KV budget, and the $100\%$ is selected to maintain the important information.
Then We aim to select dimensions that yield substantial compression benefits while avoiding severe information loss caused by compressing important tokens along those dimensions.
The optional dimensions should be powers of two, and the number of candidate options is kept limited to avoid computational fragmentation.
We find that at $50\%$ the performance degradation is noticeable but not severe, making the cost of compression difficult to distinguish; however, compressing important tokens to this dimensionality can still incur non-negligible information loss.
Therefore, we skip $50\%$ and choose $25\%$ and $12.5\%$.

%% file: Appendix/longbench.tex
\subsection{More Results on LongBench}
\label{appendix:exp:longbench}

In \cref{tab:app:longbench} we supplement the experimental results with a KV size of 256 on the LongBench. 
The results show that our solution achieves a consistent performance improvement compared with H2O, SnapKV, and PyramidKV.
In \cref{tab:app:mdh-headkv} we present the full results of HeadKV and \optalg{} on LongBench.

\begin{table*}[t]
  \caption{Additional results on LongBench.}
  \label{tab:app:longbench}
  \setlength{\tabcolsep}{1.5pt}
  \newcommand\rot[1]{\rotatebox[origin=c]{45}{#1}}
  \centering
    \footnotesize
        \begin{tabular}{lccccccccccccccccc}
          \toprule
          \multirow{2}{*}[-2.5ex]{Method} & \multicolumn{3}{c}{Single-Doc QA} & \multicolumn{3}{c}{Multi-Doc QA} & \multicolumn{3}{c}{Summarization} & \multicolumn{3}{c}{Few-shot Learning} & \multicolumn{2}{c}{Synthetic} & \multicolumn{2}{c}{Code} & \multirow{2}{*}[-2.5ex]{Avg.} \\
            \cmidrule(lr){2-4} \cmidrule(lr){5-7} \cmidrule(lr){8-10} \cmidrule(lr){11-13} \cmidrule(lr){14-15} \cmidrule(lr){16-17}
            % \cline{2-14}
          & \rot{NQA} & \rot{Qasper} & \rot{MF-en} & \rot{HQA} & \rot{2WQA} & \rot{Musq} & \rot{GRep} & \rot{QMSum} & \rot{MNews} & \rot{TREC} & \rot{TQA} & \rot{SAMSum} & \rot{PCount} & \rot{PRe} & \rot{Lcc} & \rot{RB-P} & \\
          
          \midrule
            \multicolumn{18}{c}{\small \rule{0pt}{8pt}\textit{Llama-3-8B-Instruct KV size = 256}} \\ 
        \midrule

        Full & 25.56 & 32.30 & 39.71 & 43.56 & 35.49 & 21.14 & 28.58 & 23.22 & 26.67 & 74.00 & 90.48 & 42.29 & 4.80 & 69.75 & 59.13 & 54.02 & 41.92 \\
        H2O & 23.09 & 18.56 & 29.02 & 36.01 & 28.07 & 15.70 & 23.31 & 21.48 & 24.57 & 60.50 & 88.70 & 38.55 & 5.37 & 68.04 & 58.52 & 52.73 & 37.01 \\
        SnapKV & 23.36 & 20.2 & 37.36 & 42.37 & 33.18 & 19.96 & 21.76 & 22.00 & 22.83 & 71.00 & 90.86 & 40.03 & \textbf{5.83} & \textbf{69.50} & \textbf{60.16} & \textbf{55.91} & 39.77 \\
        Pyramid & 23.99 & 20.51 & 36.06 & 42.47 & 31.34 & 20.28 & 21.36 & 22.68 & 22.83 & 71.00 & 90.48 & 39.86 & \textbf{5.83} & 69.25 & 58.64 & 54.06 & 39.42 \\
        % HKV & 24.53 & 30.52 & 38.46 & 43.99 & 35.62 & 20.60 & 23.33 & 22.85 & 24.97 & 72.50 & 90.43 & 40.13 & \textbf{6.03} & 69.75 & 62.12 & \textbf{61.14} & 41.69 \\
        MD & \textbf{25.26} & \textbf{31.39} & \textbf{39.60} & \textbf{43.82} & \textbf{35.46} & \textbf{21.59} & \textbf{27.07} & \textbf{22.95} & \textbf{26.54} & \textbf{73.00} & \textbf{90.61} & \textbf{42.40} & 5.02 & \textbf{69.50} & 59.11 & 54.84 & \textbf{41.76} \\
        % MD-H & 24.85 & 30.13 & \textbf{40.45} & \textbf{44.22} & \textbf{36.24} & 20.74 & 26.04 & \textbf{23.01} & 26.05 & \textbf{73.00} & 90.56 & 41.32 & 5.64 & 69.50 & \textbf{62.44} & 59.84 & \textbf{42.13} \\

        \midrule

        \multicolumn{18}{c}{\small \rule{0pt}{8pt}\textit{Mistral-7B-Instruct KV size = 256}} \\ 

        \midrule

        Full & 29.07 & 41.54 & 52.88 & 49.37 & 39.01 & 28.58 & 35.07 & 25.71 & 27.73 & 76.00 &	88.59 & 47.51 & 6.00 & 98.50 & 61.48 & 62.68 & 48.11 \\
        H2O & 25.93 & 34.24 & 42.34 & 46.22 & 33.65 & 22.8 & 26.90 & 21.61 & 25.89 & 72.50 & 87.69 & 40.96 & 6.50 & 88.00 & 53.94 & 47.76 & 42.31 \\
        SnapKV & 27.72 & 33.18 & 51.59 & 47.58 & 37.60 & 27.73 & 23.98 & 23.12 & 23.21 & 73.00 & 88.79 & 44.88 & 5.00 & 96.00 & 58.05 & 58.12  & 44.97 \\
        Pyramid & 27.60 & 33.36 & 52.23 & 48.48 & 38.11 & 27.10 & 23.64 & 23.20 & 23.25 & 72.50 & \textbf{89.56} & 45.14 & 5.50 & 96.00 & 56.90 & 55.62 & 44.89 \\
        % HKV & 28.62 & 39.27 & 52.56 & \textbf{50.78} & \textbf{38.80} & \textbf{28.56} & 28.49 & 24.38 & 26.23 & \textbf{75.50} & 89.52 & 44.93 & 5.50 & 97.00 & 60.41 & 61.43 & 47.00 \\
        MD & \textbf{29.62} & \textbf{41.50} & \textbf{53.05} & \textbf{48.49} & \textbf{38.70} & \textbf{28.16} & \textbf{30.67} & \textbf{25.33} & \textbf{27.59} & \textbf{75.00} & 88.75 & \textbf{47.81} & \textbf{6.00} & \textbf{98.00} & \textbf{61.44} & \textbf{62.30} & \textbf{47.65} \\
        % MD-H & 29.34 & 39.96 & \textbf{53.14} & 49.60 & 38.78 & 28.32 & \textbf{30.84} & 24.73 & 27.30 & 75.00 & \textbf{89.86} & 46.41 & 5.50 & 97.50 & \textbf{61.66} & \textbf{62.65} & 47.54 \\
            
          \bottomrule
        \end{tabular}

\end{table*}

\begin{table*}[t]
    \caption{Full LongBench results for HeadKV and \optalg{} (MD-H).}
  \label{tab:app:mdh-headkv}
  \setlength{\tabcolsep}{1.5pt}
  \newcommand\rot[1]{\rotatebox[origin=c]{45}{#1}}
  \centering
    \footnotesize
        \begin{tabular}{lccccccccccccccccc}
          \toprule
          \multirow{2}{*}[-2.5ex]{Method} & \multicolumn{3}{c}{Single-Doc QA} & \multicolumn{3}{c}{Multi-Doc QA} & \multicolumn{3}{c}{Summarization} & \multicolumn{3}{c}{Few-shot Learning} & \multicolumn{2}{c}{Synthetic} & \multicolumn{2}{c}{Code} & \multirow{2}{*}[-2.5ex]{Avg.} \\
            \cmidrule(lr){2-4} \cmidrule(lr){5-7} \cmidrule(lr){8-10} \cmidrule(lr){11-13} \cmidrule(lr){14-15} \cmidrule(lr){16-17}
            % \cline{2-14}
          & \rot{NQA} & \rot{Qasper} & \rot{MF-en} & \rot{HQA} & \rot{2WQA} & \rot{Musq} & \rot{GRep} & \rot{QMSum} & \rot{MNews} & \rot{TREC} & \rot{TQA} & \rot{SAMSum} & \rot{PCount} & \rot{PRe} & \rot{Lcc} & \rot{RB-P} & \\
          
          \midrule
            \multicolumn{18}{c}{\small \rule{0pt}{8pt}\textit{Llama-3-8B-Instruct KV size = 128}} \\ 
        \midrule

        Full & 25.56 & 32.30 & 39.71 & 43.56 & 35.49 & 21.14 & 28.58 & 23.22 & 26.67 & 74.00 & 90.48 & 42.29 & 4.80 & 69.75 & 59.13 & 54.02 & 41.92 \\
        HeadKV & 22.56 & 28.04 & 39.64 & 44.40 & 31.83 & 20.62 & 22.16 & 23.04 & 24.04 & 71.00 & 90.10 & 38.81 & 4.89 & 68.75 & 61.72 & 61.77 & 40.84 \\
        MD-H & 24.59 & 27.15 & 41.39 & 43.59 & 35.10 & 20.59 & 23.43 & 22.74 & 25.56 & 70.00 & 90.34 & 40.51 & 5.08 & 69.50 & 62.30 & 59.69 & 41.35 \\

        \midrule
            \multicolumn{18}{c}{\small \rule{0pt}{8pt}\textit{Llama-3-8B-Instruct KV size = 256}} \\ 
        \midrule

        Full & 25.56 & 32.30 & 39.71 & 43.56 & 35.49 & 21.14 & 28.58 & 23.22 & 26.67 & 74.00 & 90.48 & 42.29 & 4.80 & 69.75 & 59.13 & 54.02 & 41.92 \\
        HeadKV & 24.53 & 30.52 & 38.46 & 43.99 & 35.62 & 20.60 & 23.33 & 22.85 & 24.97 & 72.50 & 90.43 & 40.13 & 6.03 & 69.75 & 62.12 & 61.14 & 41.69 \\
        MD-H & 24.85 & 30.13 & 40.45 & 44.22 & 36.24 & 20.74 & 26.04 & 23.01 & 26.05 & 73.00 & 90.56 & 41.32 & 5.64 & 69.50 & 62.44 & 59.84 & 42.13 \\

        \midrule
            \multicolumn{18}{c}{\small \rule{0pt}{8pt}\textit{Llama-3-8B-Instruct KV size = 512}} \\ 
        \midrule

        Full & 25.56 & 32.30 & 39.71 & 43.56 & 35.49 & 21.14 & 28.58 & 23.22 & 26.67 & 74.00 & 90.48 & 42.29 & 4.80 & 69.75 & 59.13 & 54.02 & 41.92 \\
        HeadKV & 24.83 & 29.85 & 38.06 & 44.30 & 36.34 & 22.17 & 24.70 & 23.14 & 26.03 & 73.50 & 90.56 & 40.82 & 5.66 & 69.50 & 62.36 & 60.66 & 42.03 \\
        MD-H & 25.37 & 31.96 & 39.82 & 43.97 & 36.91 & 21.39 & 27.61 & 23.35 & 26.24 & 73.50 & 90.64 & 41.83 & 5.27 & 69.50 & 62.02 & 58.29 & 42.35 \\

        \midrule

        \multicolumn{18}{c}{\small \rule{0pt}{8pt}\textit{Mistral-7B-Instruct KV size = 128}} \\ 

        \midrule

        Full & 29.07 & 41.54 & 52.88 & 49.37 & 39.01 & 28.58 & 35.07 & 25.71 & 27.73 & 76.00 &	88.59 & 47.51 & 6.00 & 98.50 & 61.48 & 62.68 & 48.11 \\
        HeadKV & 27.23 & 37.88 & 52.46 & 47.74 & 38.95 & 27.69 & 26.15 & 24.12 & 24.32 & 75.00 & 90.11 & 44.35 & 4.50 & 94.5 & 58.93 & 57.95 & 45.74 \\
        MD-H & 28.22 & 38.08 & 51.12 & 48.79 & 37.75 & 27.32 & 27.38 & 24.73 & 26.63 & 66.50 & 90.11 & 44.95 & 5.23 & 98.00 & 61.14 & 60.38 & 46.02 \\

        \midrule

        \multicolumn{18}{c}{\small \rule{0pt}{8pt}\textit{Mistral-7B-Instruct KV size = 256}} \\ 

        \midrule

        Full & 29.07 & 41.54 & 52.88 & 49.37 & 39.01 & 28.58 & 35.07 & 25.71 & 27.73 & 76.00 &	88.59 & 47.51 & 6.00 & 98.50 & 61.48 & 62.68 & 48.11 \\
        HeadKV & 28.62 & 39.27 & 52.56 & 50.78 & 38.80 & 28.56 & 28.49 & 24.38 & 26.23 & 75.50 & 89.52 & 44.93 & 5.50 & 97.00 & 60.41 & 61.43 & 47.00 \\
        MD-H & 29.34 & 39.96 & 53.14 & 49.60 & 38.78 & 28.32 & 30.84 & 24.73 & 27.30 & 75.00 & 89.86 & 46.41 & 5.50 & 97.50 & 61.66 & 62.65 & 47.54 \\

        \midrule

        \multicolumn{18}{c}{\small \rule{0pt}{8pt}\textit{Mistral-7B-Instruct KV size = 512}} \\ 

        \midrule

        Full & 29.07 & 41.54 & 52.88 & 49.37 & 39.01 & 28.58 & 35.07 & 25.71 & 27.73 & 76.00 &	88.59 & 47.51 & 6.00 & 98.50 & 61.48 & 62.68 & 48.11 \\
        HeadKV & 27.23 & 37.88 & 52.46 & 47.74 & 38.95 & 27.69 & 26.15 & 24.12 & 24.32 & 75.00 & 90.11 & 44.35 & 4.50 & 94.5 & 58.93 & 57.95 & 45.74 \\
        MD-H & 28.22 & 38.08 & 51.12 & 48.79 & 37.75 & 27.32 & 27.38 & 24.73 & 26.63 & 66.50 & 90.11 & 44.95 & 5.23 & 98.00 & 61.14 & 60.38 & 46.02 \\
            
          \bottomrule
        \end{tabular}

\end{table*}

%% file: Appendix/ruler.tex
\subsection{More Results on RULER}
\label{appendix:exp:ruler}

In \cref{tab:app:ruler} we present the full results on RULER.

\begin{table}[t]
  \caption{Full results of our solution and baselines on RULER.}
  \label{tab:app:ruler}
  \begin{center}
    \begin{small}
        \begin{tabular}{lcccccccccc}
          \toprule
          Method & S1 & S2 & MK1 & MK2 & MQ & MV & VT & QA-1 & QA-2 & Avg. \\

          \midrule
            \multicolumn{11}{c}{\small \rule{0pt}{8pt}\textit{Context Length = 8K}} \\ 
        \midrule

        Full & 100 & 100 & 99.4 & 99.4 & 99.25 & 83.4  & 96.64 & 69.83 & 58.4 & 89.59 \\
        H2O & 1.8 & 0.2 & 0.4 & 3 & 0.05 & 0.3 & 15.76 & 36.8 & 32.8 & 10.12\\
        SnapKV & \textbf{100} & 98.2 & 96.8 & 92 & 75.95 & 71.9 & 80.52 & 63.93 & 50 & 81.03\\
        PyramidKV & \textbf{100} & 98 & 96.8 & 93 & 74.6 & 72.05 & 84.16 & 60.28 & 50 & 80.99\\
        HeadKV & \textbf{100} & 99.2 & \textbf{99.2} & \textbf{99} & \textbf{98.85} & \textbf{94.75} & 89.92 & 65.6 & 55.2 & 89.08 \\
        \basicalg{} & 99.6 & 99.4 & 98.8 & 95.6 & 88.4 & 80.25 & \textbf{97.52} & 67.62 & 57.4 & 87.18\\
        \optalg{} & \textbf{100} & \textbf{99.6} & 98.4 & 97.4 & 96.5 & 89.7 & 97.28 & \textbf{68.78} & \textbf{57.6} & \textbf{89.47}\\

        \midrule

        \multicolumn{11}{c}{\small \rule{0pt}{8pt}\textit{Context Length = 32K}} \\ 
        \midrule

        Full & 99.60 & 	98 & 98 & 97 & 46.6 & 94.15 & 22.28 & 69.78 & 53.2 & 75.40 \\
        H2O & 0.8 & 0 & 0 & 0 & 0 & 0 & 4.04 & 24.2 & 23.6 & 5.85 \\
        SnapKV & \textbf{100} & 85.8 & 89.8 & 13.6 & 7.25 & 16.5 & 20.88 & 62.77 & 45.8 & 49.16\\
        PyramidKV & \textbf{100} & 80.6 & 78.8 & 14.4 & 3.69 & 5.05 & 20.48 & 60.27 & 44.6 & 45.31 \\
        HeadKV & \textbf{100} & 90.6 & 92 & \textbf{44} & 24.8 & 56.5 & 20.92 & 64.27 & 50.8 & 60.43 \\
        \basicalg{} & 98.40 & 85.60 & 92.4 & 8.4 & 7.75 & 17.1 & 21.24 & \textbf{66.47} & 51.2 & 49.84\\
        \optalg{} & \textbf{100} & \textbf{95.2} & \textbf{96} & 38.4 & \textbf{31.05} & \textbf{64.35} & \textbf{21.6} & 65.80 & \textbf{51.4} & \textbf{62.64}\\

        % \multicolumn{11}{c}{\small \rule{0pt}{8pt}\textit{Mistral-7B-Instruct}} \\ 

        % \midrule

        % Full & \\
        % H2O & \\
        % SnapKV & \\
        % PyramidKV & \\
        % Palu & \\
        % HeadKV & \\
        % MD-SnapKV & \\
        % MD-HeadKV & \\
            
          \bottomrule
        \end{tabular}
    \end{small}
  \end{center}
  % \vskip -0.1in
\end{table}

%% file: Appendix/palu.tex
\subsection{Comparison with Dimension Reduction Methods}
\label{appendix:exp:palu}

We benchmark our approach against dimension reduction-based methods to provide a comprehensive evaluation. 
We select \textbf{Palu}~\cite{palu} as the representative baseline for this category due to its clear implementation and reproducibility.
Unlike eviction strategies that discard tokens, Palu performs low-rank decomposition on the linear projection matrices ($W_k, W_v$), effectively reducing the hidden dimension of the KV cache while retaining all tokens.

\paragraph{Setup.}
To ensure a fair comparison, we align the total memory footprint of the KV cache for both methods. 
Specifically, we set the token budget of our \basicalg{} and \optalg{} variants to $10\%$ of the prefill sequence length. 
Correspondingly, we adjust the projection dimension in Palu to a specific ratio that yields an identical KV cache size (i.e., a $40\%$ projection dimension).

\paragraph{Results.}
As shown in \cref{tab:palu}, Palu experiences a significant performance degradation at this compression level. In contrast, our solution demonstrates superior resilience, maintaining performance comparable to the full-cache baseline.

\begin{table*}[t]
  \caption{Comparative analysis with PALU on LongBench.}
  \label{tab:palu}
  \setlength{\tabcolsep}{1.5pt}
  \newcommand\rot[1]{\rotatebox[origin=c]{45}{#1}}
  \centering
  \footnotesize
    \begin{tabular}{lccccccccccccccccc}
      \toprule
      \multirow{2}{*}[-2.5ex]{Method} & \multicolumn{3}{c}{Single-Doc QA} & \multicolumn{3}{c}{Multi-Doc QA} & \multicolumn{3}{c}{Summarization} & \multicolumn{3}{c}{Few-shot Learning} & \multicolumn{2}{c}{Synthetic} & \multicolumn{2}{c}{Code} & \multirow{2}{*}[-2.5ex]{Avg.} \\
        \cmidrule(lr){2-4} \cmidrule(lr){5-7} \cmidrule(lr){8-10} \cmidrule(lr){11-13} \cmidrule(lr){14-15} \cmidrule(lr){16-17}
      & \rot{NQA} & \rot{Qasper} & \rot{MF-en} & \rot{HQA} & \rot{2WQA} & \rot{Musq} & \rot{GRep} & \rot{QMSum} & \rot{MNews} & \rot{TREC} & \rot{TQA} & \rot{SAMSum} & \rot{PCount} & \rot{PRe} & \rot{Lcc} & \rot{RB-P} & \\
      
      % \midrule
      % \multicolumn{18}{c}{\small \rule{0pt}{8pt}\textit{Llama-3-8B-Instruct Budget = 20\%}} \\ 
      % \midrule

    % Full & 25.56 & 32.30 & 39.71 & 43.56 & 35.49 & 21.14 & 28.58 & 23.22 & 26.67 & 74.00 & 90.48 & 42.29 & 4.80 & 69.75 & 59.13 & 54.02 & 41.92 \\
    % Palu & 24.31 & \textbf{32.67} & 36.02 & 36.37 & 26.81 & 18.26 & \textbf{28.65} & \textbf{23.78} & \textbf{26.63} & \textbf{74.00} & 89.79 & 42.18 & 4.50 & 67.00 & 54.94 & 49.40 & 39.71 \\
    % MD & \textbf{25.59} & 30.65 & \textbf{39.99} & 43.59 & 35.08 & \textbf{21.29} & 28.27 & 23.43 & 25.87 & 73.50 & 90.48 & \textbf{42.65} & \textbf{4.90} & \textbf{69.25} & 59.25 & 53.99 & 41.74 \\
    % MD-Head & 25.11 & 32.19 & 39.58 & \textbf{43.69} & \textbf{35.55} & 21.02 & 28.39 & 23.12 & 26.18 & 73.50 & \textbf{90.56} & 42.05 & 4.66 & \textbf{69.25} & \textbf{61.29} & \textbf{56.59} & \textbf{42.05} \\

    \midrule
    \multicolumn{18}{c}{\small \rule{0pt}{8pt}\textit{Llama-3-8B-Instruct Budget = 10\%}} \\ 
    \midrule

    Full & 25.56 & 32.30 & 39.71 & 43.56 & 35.49 & 21.14 & 28.58 & 23.22 & 26.67 & 74.00 & 90.48 & 42.29 & 4.80 & 69.75 & 59.13 & 54.02 & 41.92 \\
    Palu & 15.44 & 9.32 & 19.22 & 17.82 & 16.87 & 8.65 & 9.74 & 20.21 & 14.08 & 45.00 & 22.89 & 23.32 & 1.59 & 4.35 & 7.82 & 9.07 & 15.34 \\
    MD & 25.22 & 28.55 & 38.67 & 43.39 & 32.42 & 21.13 & 25.72 & 23.12 & \textbf{25.96} & 72.00 & 90.48 & \textbf{42.74} & \textbf{4.90} & 69.25 & 59.16 & 54.23 & 41.06 \\
    MD-H & \textbf{25.32} & \textbf{32.07} & \textbf{39.96} & \textbf{43.91} & \textbf{36.38} & \textbf{21.48} & \textbf{27.39} & \textbf{23.25} & 25.95 & \textbf{73.50} & \textbf{90.56} & 41.84 & 4.81 & \textbf{69.50} & \textbf{61.86} & \textbf{57.19} & \textbf{42.19} \\

    % \midrule
    % \multicolumn{18}{c}{\small \rule{0pt}{8pt}\textit{Mistral-7B-Instruct Budget = 20\%}} \\ 
    % \midrule

    % Full & 29.07 & 41.54 & 52.88 & 49.37 & 39.01 & 28.58 & 35.07 & 25.71 & 27.73 & 76.00 & 88.59 & 47.51 & 6.00 & 98.50 & 61.48 & 62.68 & 48.11 \\
    % Palu & \textbf{29.19} & 40.54 & \textbf{53.67} & \textbf{51.31} & \textbf{40.01} & \textbf{29.95} & 34.13 & \textbf{25.73} & \textbf{27.60} & 73.50 & 87.49 & 46.09 & 5.00 & 95.50 & 58.59 & 60.54 & 47.43 \\
    % MD & \textbf{29.19} & \textbf{41.76} & 53.23 & 49.44 & 39.21 & 28.45 & 33.53 & 25.40 & 27.22 & \textbf{76.00} & 88.59 & 47.17 & 5.50 & \textbf{99.00} & 61.58 & \textbf{62.66} & 48.00 \\
    % MD-Head & 29.15 & 40.79 & 53.66 & 49.78 & 39.01 & 28.57 & \textbf{34.64} & 25.55 & 27.51 & \textbf{76.00} & \textbf{88.64} & \textbf{47.68} & \textbf{6.00} & 98.50 & \textbf{61.51} & 62.60 & \textbf{48.10} \\

    \midrule
    \multicolumn{18}{c}{\small \rule{0pt}{8pt}\textit{Mistral-7B-Instruct Budget = 10\%}} \\ 
    \midrule

    Full & 29.07 & 41.54 & 52.88 & 49.37 & 39.01 & 28.58 & 35.07 & 25.71 & 27.73 & 76.00 & 88.59 & 47.51 & 6.00 & 98.50 & 61.48 & 62.68 & 48.11 \\
    Palu & 22.12 & 31.68 & 42.54 & 37.03 & 28.45 & 20.08 & 18.66 & 22.90 & 23.98 & 70.00 & 80.96 & 35.66 & 4.50 & 22.00 & 33.72 & 33.28 & 32.97 \\
    MD & 29.12 & 38.12 & 52.89 & 49.01 & \textbf{39.48} & 28.24 & 30.81 & 25.60 & 25.89 & 74.00 & 88.59 & \textbf{47.30} & \textbf{6.00} & 98.50 & \textbf{61.75} & 62.54 & 47.37 \\
    MD-H & \textbf{29.16} & \textbf{41.00} & \textbf{53.84} & \textbf{49.58} & 38.96 & \textbf{28.48} & \textbf{34.10} & \textbf{25.78} & \textbf{27.03} & \textbf{76.00} & \textbf{88.59} & 47.18 & \textbf{6.00} & \textbf{98.50} & 61.56 & \textbf{62.54} & \textbf{48.02} \\

      \bottomrule
    \end{tabular}
\end{table*}

%% file: Appendix/niah.tex
\subsection{Needle-in-a-Haystack tests on Mistral}
\label{appendix:exp:niah-mistral}

\begin{figure*}[t]
    \centering

    \begin{subfigure}{0.49\linewidth}
        \centering
        \includegraphics[width=\linewidth]{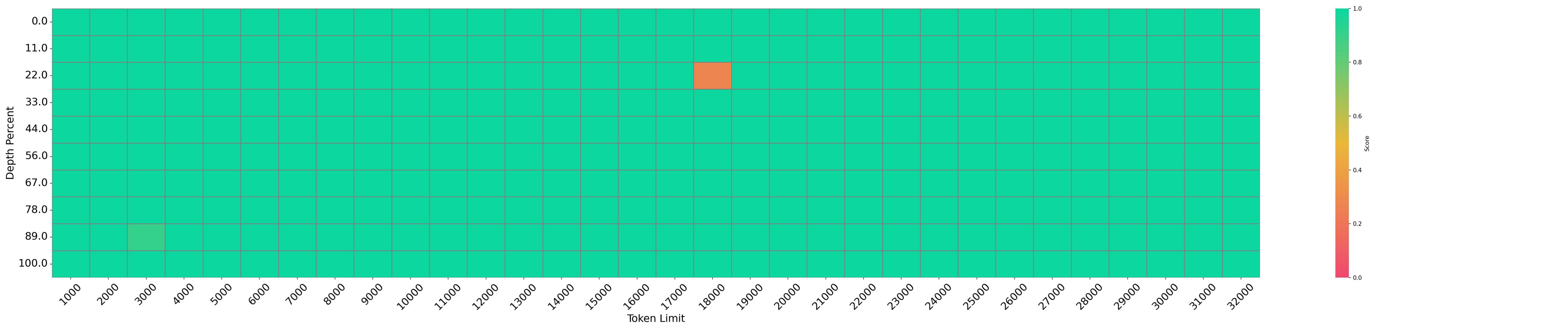}
        \caption{FullKV Score: 99.70}
    \end{subfigure}
    \begin{subfigure}{0.49\linewidth}
        \centering
        \includegraphics[width=\linewidth]{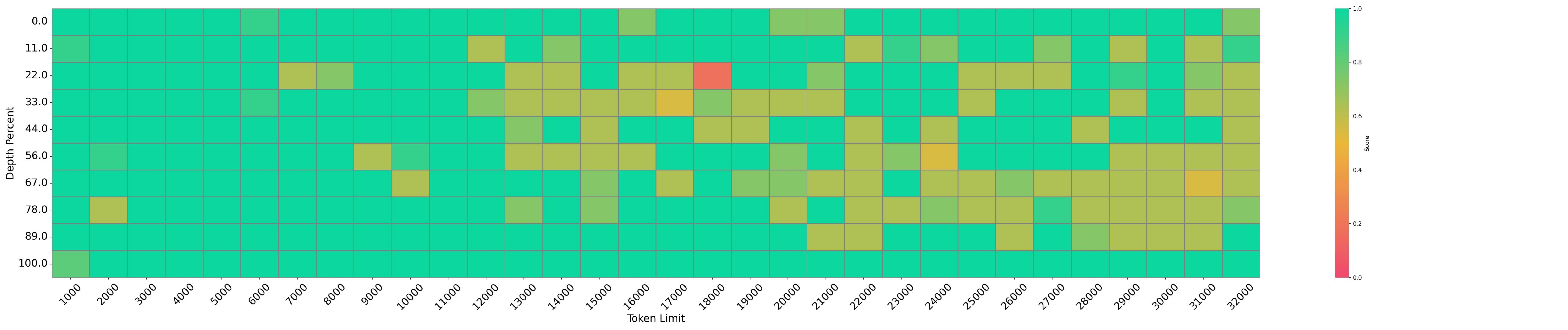}
        \caption{SnapKV Score: 89.20}
    \end{subfigure}

    \begin{subfigure}{0.49\linewidth}
        \centering
        \includegraphics[width=\linewidth]{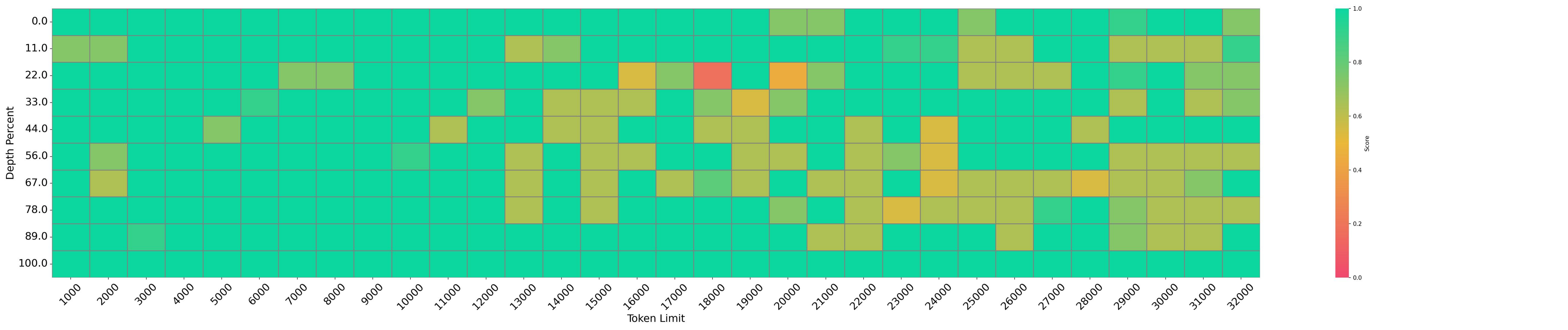}
        \caption{PyramidKV Score: 89.70}
    \end{subfigure}
    \begin{subfigure}{0.49\linewidth}
        \centering
        \includegraphics[width=\linewidth]{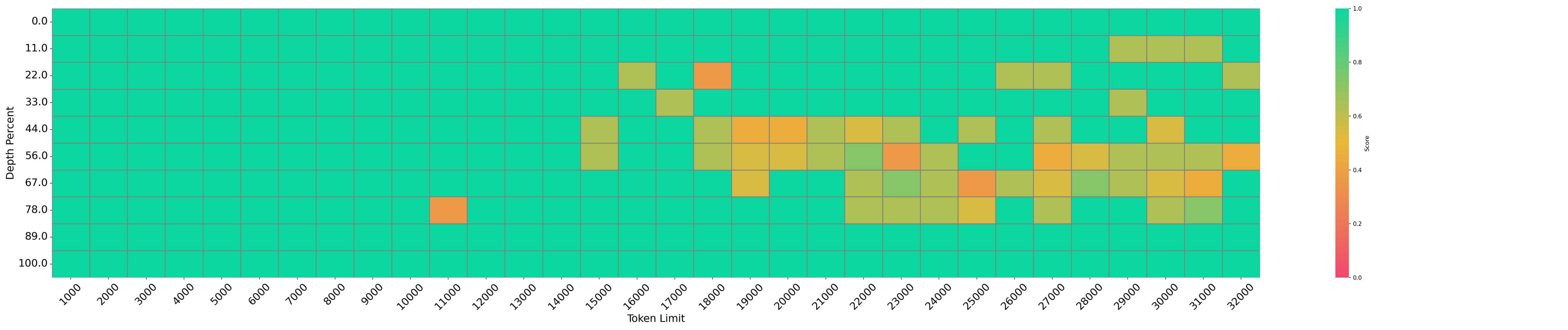}
        \caption{HeadKV Score: 93.20}
    \end{subfigure}

    \begin{subfigure}{0.49\linewidth}
        \centering
        \includegraphics[width=\linewidth]{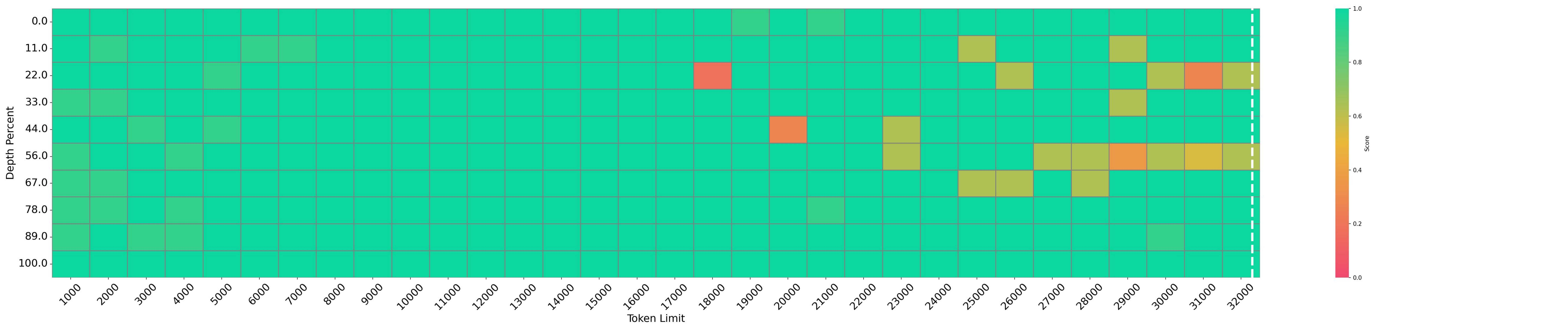}
        \caption{\basicalg{} Score: 96.60}
    \end{subfigure}
    \begin{subfigure}{0.49\linewidth}
        \centering
        \includegraphics[width=\linewidth]{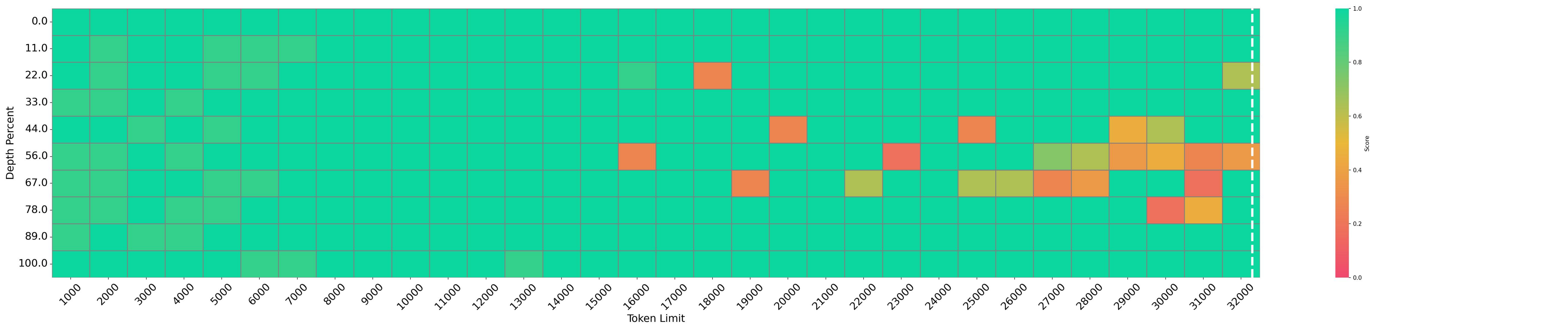}
        \caption{\optalg{} Score: 94.90}
    \end{subfigure}

    \caption{Results of Needle-in-a-Haystack test on Mistral-7B-Instruct with KV size $128$.}
    \label{fig:app:niah:mistral}
\end{figure*}

We evaluate the Needle-in-a-Haystack test on the Mistral-7B-Instruct-v0.2 model.
We set the maximum context length to 32K, the KV size to $128$, and the query window size to $16$.
As shown in \cref{fig:app:niah:mistral}, \basicalg{} significantly outperforms the baselines that does not relies on additional head-level information (SnapKV, PyramidKV).
Equipped with the same head-level information, \optalg{} also outperforms HeadKV.
We further find that in Mistral \basicalg{} achieves a better score compared with HeadKV and \optalg{}, which indicates that HeadKV still has limitations in evaluating the importance of attention heads. 
% We attribute this deficiency to the fact that HeadKV measures head importance via a needle-in-a-haystack evaluation: although it conducts comprehensive assessments using extensive datasets, it fails to identify which heads are more critical within the current KV cache context. 
% In contrast, our method can partially reveal the relative importance of heads in the current KV cache. A more promising strategy would be to combine the strengths of both approaches, which we leave as future work.

%% file: Appendix/allocation.tex
\subsection{More Results on Dimension Allocation}
\label{appendix:exp:dim-alloc}

For most experiments, the candidate set of dimension compression ratios is set to $\{0\%, 12.5\%, 25\%, 100\%\}$.
We report the average token length and the proportion of tokens allocated to each compression ratio, averaged over all layers and samples in each dataset.
\cref{fig:app:dim-alloc} presents the results on different LongBench datasets.
We find that all four compression ratios take effect by allocating a portion of tokens.
Specifically, tokens are predominantly allocated to the $0\%$ and $100\%$ compression ratios, while the two intermediate ratios account for a non-negligible 
$10\%$-$15\%$ of the total tokens, which improves the space utilization of the KV cache.
Notably, in longer datasets, an even larger proportion of tokens are assigned to the $0\%$ ratio, which is a direct consequence of the constrained budget setting.

\begin{figure}[t]
    \centering

    \begin{subfigure}{0.48\linewidth}
        \centering
        \includegraphics[width=\linewidth]{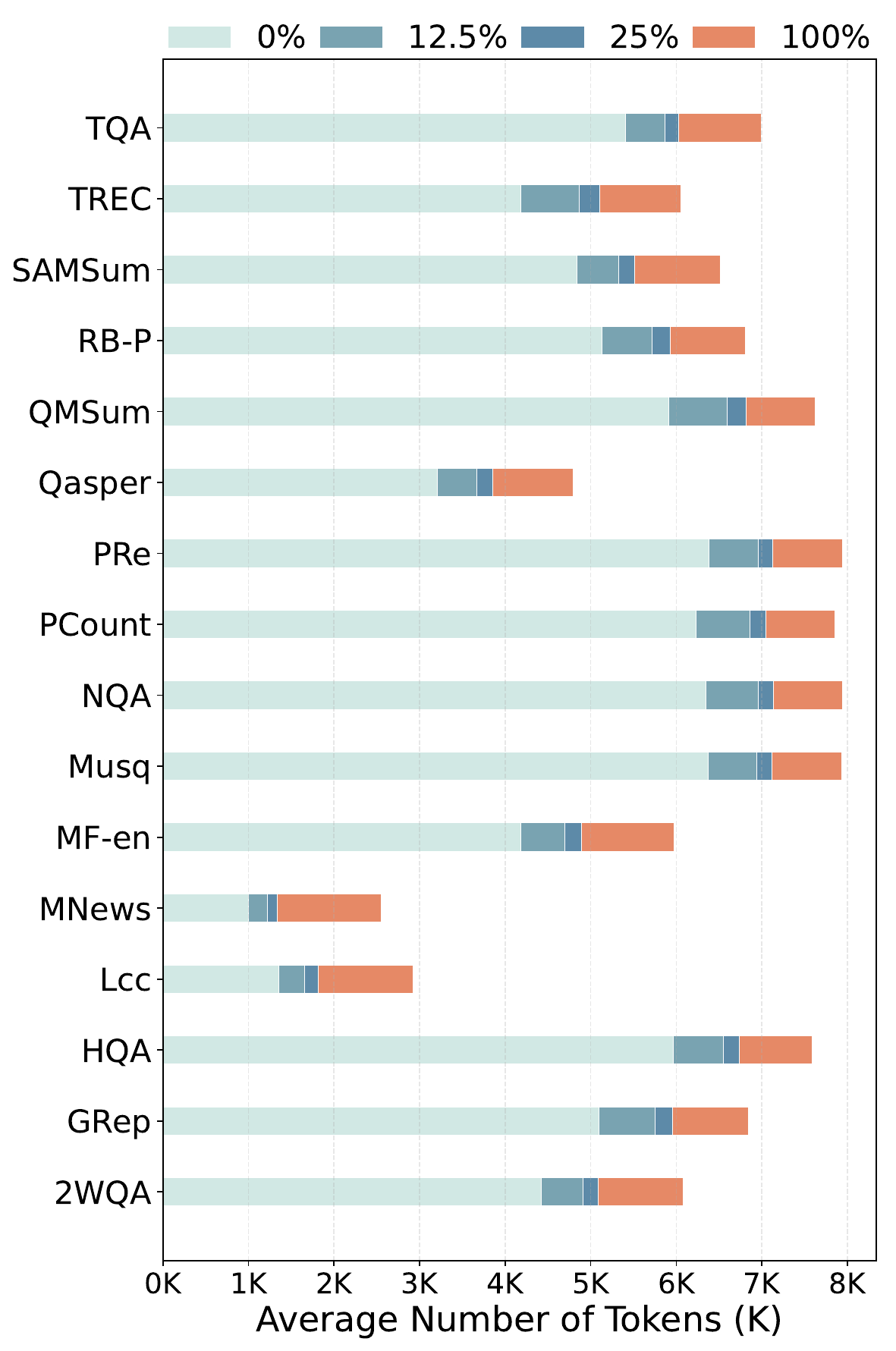}
        \caption{KV size = $256$}
        \label{fig:app:dim-alloc:256}
    \end{subfigure} \hfill
    \begin{subfigure}{0.48\linewidth}
        \centering
        \includegraphics[width=\linewidth]{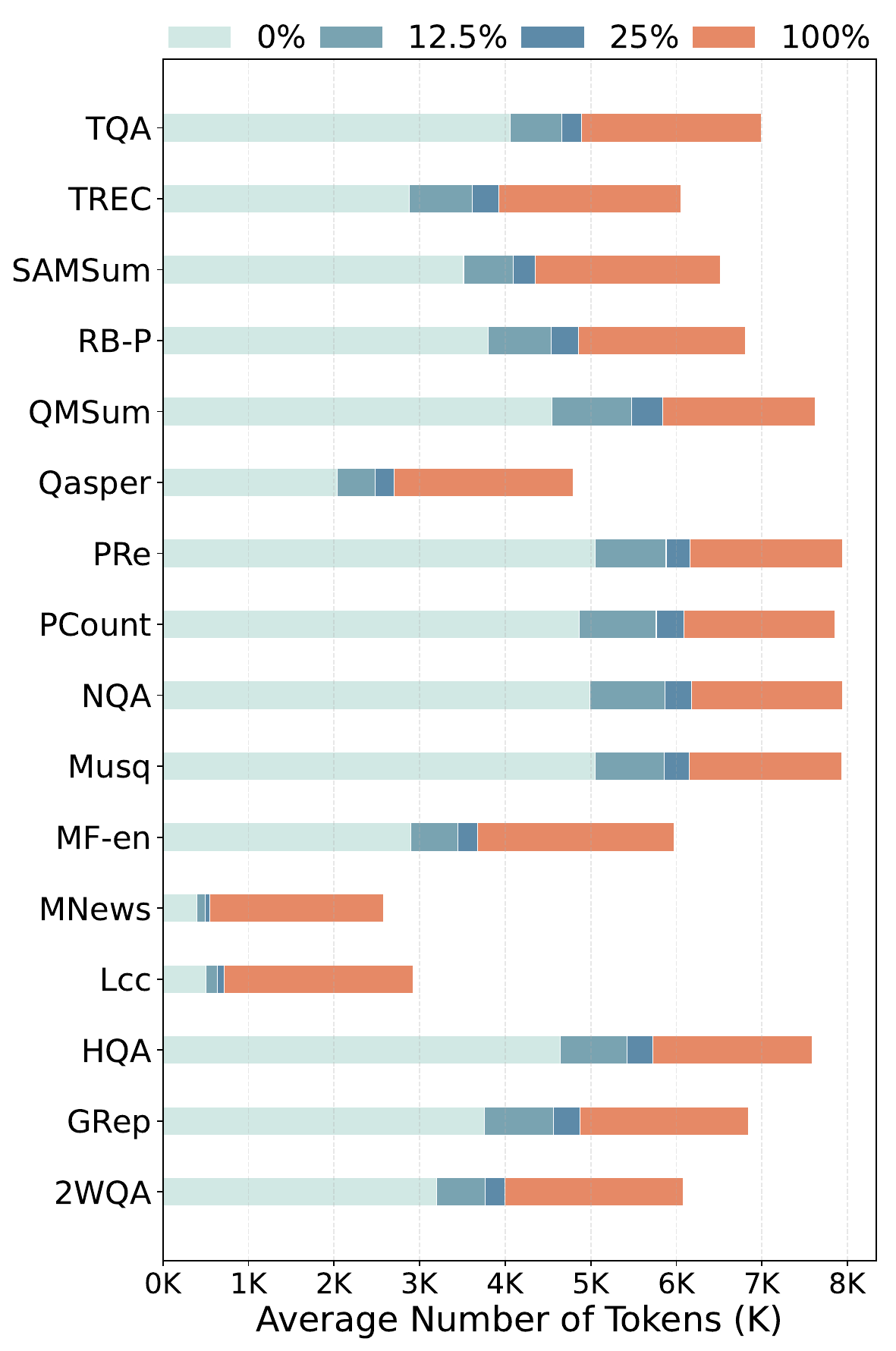}
        \caption{KV size = $512$}
        \label{fig:app:dim-alloc:512}
    \end{subfigure}

    \caption{Average number of tokens and the allocation statistics across candidate compression ratios on LongBench.}
    \label{fig:app:dim-alloc}
\end{figure}